\documentclass[10pt,twocolumn,letterpaper]{article}

\usepackage{cvpr}
\usepackage{times}
\usepackage{epsfig}
\usepackage{graphicx}
\usepackage{amsmath,bm}
\usepackage{amssymb}
\usepackage{subfigure}
\usepackage{booktabs}
\usepackage{diagbox}
\usepackage{multirow}
\usepackage{caption}

\usepackage{amsthm}
\usepackage{mathrsfs}
\usepackage{algorithm}
\usepackage{algorithmic}
\usepackage{sidecap}
\usepackage{url}
\usepackage{color}
\usepackage{mathrsfs}
\usepackage{pdfpages}

\newtheorem{prop}{Proposition}
\theoremstyle{definition}
\newtheorem{defn}{Definition}[section]

\newcommand{\RNum}[1]{\uppercase\expandafter{\romannumeral #1\relax}}

\usepackage[colorlinks, citecolor=blue, bookmarks=false]{hyperref}

\cvprfinalcopy 

\def\cvprPaperID{1149} 

\begin{document}

\title{Zero-Assignment Constraint for Graph Matching with Outliers}

\author{{Fudong Wang$^1$},
        {Nan Xue$^1$},
        {Jin-Gang Yu$^2$},
        {Gui-Song Xia$^1\thanks{Corresponding author}$}\\
{$^1$Wuhan University, China}\\
{\tt\small\{fudong-wang, xuenan, guisong.xia\}@whu.edu.cn}\\
{$^2$South China University of Technology, China}\\
{\tt\small jingangyu@scut.edu.cn}
}

\maketitle

\begin{abstract}
	Graph matching (GM), as a longstanding problem in computer vision and pattern recognition, still suffers from numerous cluttered outliers in practical applications. To address this issue, we present the zero-assignment constraint (ZAC) for approaching the graph matching problem in the presence of outliers. The underlying idea is to suppress the matchings of outliers by assigning zero-valued vectors to the potential outliers in the obtained optimal correspondence matrix. We provide elaborate theoretical analysis to the problem, {\em \ie}, GM with ZAC, and figure out that the GM problem with and without outliers are intrinsically different, which enables us to put forward a sufficient condition to construct valid and reasonable objective function. Consequently, we design an efficient outlier-robust algorithm to significantly reduce the incorrect or redundant matchings caused by numerous outliers.
    Extensive experiments demonstrate that our method can achieve the state-of-the-art performance in terms of accuracy and efficiency, especially in the presence of numerous outliers.
\end{abstract}

\section{Introduction}\label{sec:introduction}
In many real applications of computer vision and pattern recognition, the feature sets of interest represented as graphs are usually cluttered with numerous outliers~\cite{[2002-Belongie-pami],[Zeng2010],[2012-Yao-eccv],[2016-Shen-eccv]}, which often reduce the accuracy of GM. Although recent works on GM~\cite{[2010-Cho-eccv],[2013-Egozi],[2017-Huu-cvpr],[2011-Lee-cvpr],[2019-FRGM],[2016-Zhou-pami]} can achieve satisfactory results for simple graphs that consist of only inliers or a few outliers, they still lack of ability to tolerate numerous outliers arising in complicated graphs. Empirically, the inliers in one graph are nodes that have highly-similar corresponding nodes in the other graph, while the outliers do not. Based on the empirical criterion, the aforementioned methods hope to match inliers to inliers correctly and force outliers to only match outliers. However, due to the complicated mutual relationships between inliers and outliers, they usually result in incorrect matchings between inliers or redundant matchings between outliers  (\eg, Fig.~\ref{fig:fig1_com} (a)).

\begin{figure}
	\centering
	\subfigure[{\bf Left}: incorrect/redundant matchings (lines in red) caused by outliers. {\bf Right}:  generated (yellow) {\itshape v.s.} the ideal (red) correspondence matrix.]
	{\includegraphics[width=0.99\linewidth]{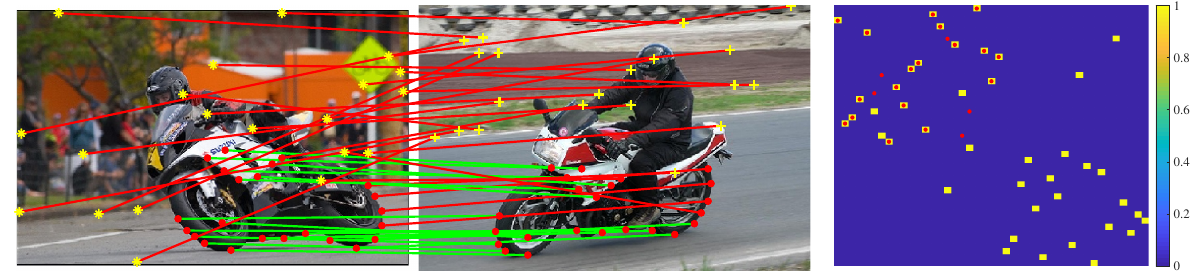}}
	
	\vspace{-3mm}
	\subfigure[{\bf Left}: our graph matching result. {\bf Right}: our correspondence matrix with zero-assignment constraint of outliers.]
	{\includegraphics[width=0.99\linewidth]{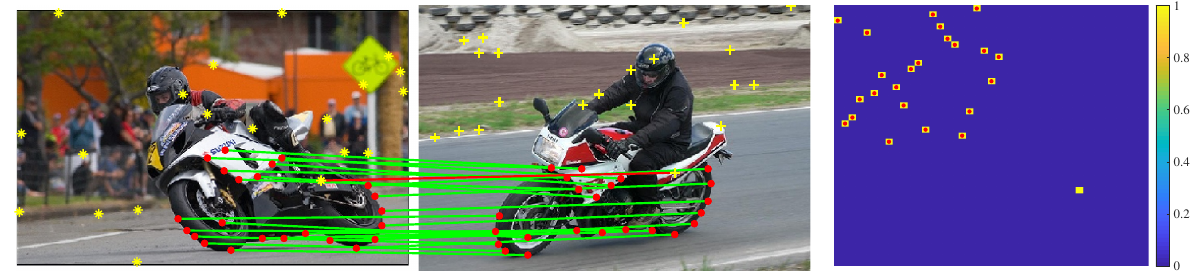}}
	\vspace{-3mm}
	\caption{{\bf ZAC} for graph matching in the presence of outliers. To suppress the undesired matchings of outliers in (a), we aim to assign the potential outliers with zero-valued vectors in our optimal correspondence matrix in (b), based on which we can both establish a theoretical foundation for graph matching with outliers and put forward an outlier identification approach that can significantly reduce incorrect or redundant matches caused by outliers in practice.}
	\label{fig:fig1_com}
\end{figure}

In this paper, we are motivated to address this challenge by introducing the zero-assignment constraint for outliers: unlike the previous methods that hope to match outliers only to outliers, it's more reasonable to suppress the matchings of outliers. Equivalently, we try to assign each potential outlier with a zero-valued vector ({\ie}, the zero-assignment constraint for outliers) in the solution of our objective function ({\eg}, the correspondence matrix in Fig.~\ref{fig:fig1_com} (b)). 

To make our idea more reasonable and practical, we try our efforts in two aspects. First, based on the zero-assignment constraint, we establish the theoretical bases including the formulation of inliers and outliers and the quantitative distinguishability between them, and then find out a sufficient condition such that the proposed objective function can only achieve its minimum at the ideal matching. Moreover, it also helps to demonstrate the intrinsic differences between GM with and without numerous outliers. Second, we propose an efficient GM algorithm consisting of fast optimization and explicit outlier identification. The optimization algorithm is modified based on the Frank-Wolfe method~\cite{[2015-Simon-nips]} combined with the k-cardinality linear assignment problem~\cite{[1997-DellAmico]} and has low space and time complexity. And then, the zero-assignment vectors in the optimal solution of our objective function can be used to assign the nodes in two graphs with joint probabilities, which measure whether the nodes are inliers or outliers and help to identify and remove the potential outliers in practice.


Our main contributions are summarized as follows:
\begin{itemize}
	\item[-] We establish the theoretical foundations for GM problem with outliers based on the zero-assignment constraint and elaborate quantitative analyses of inliers and outliers, on which bases we can theoretically put forward a sufficient condition to guide us how to construct valid and reasonable objective function.
	
	\item[-] We present an efficient GM algorithm with low space and time complexity by avoiding using the costly affinity matrix  and designing fast optimization algorithm. Combined with our outlier identification approach, we can achieve state-of-the-art performance for complicated graphs cluttered with numerous outliers.
\end{itemize}


\section{Related Work}

Known to be NP-complete~\cite{[1979-Garey],[1957-Koopmans],[1963-Lawler]}, the GM problem can only be solved in polynomial time with approximate solutions. Over the past decades, a myriad of literature have been extensively studied (see~\cite{[2004-Conte-IJPRAI],[2016-Yan-ICMR]} for surveys), we discuss the most related works in the following aspects.

{\bf Robustness to outliers}. 
The dual decomposition approach~\cite{[2013-Torresani-pami]} constructed a penalty potential in the objective function for unmatched features. The max-pooling-based method~\cite{[2014-Cho-cvpr]} was proposed to  avoid the adverse effect of false matches of outliers. A domain adaptation-based outlier-removal strategy proposed in~\cite{[2019-FRGM]} aimed to remove outliers as a pre-processing step. However, they directly rely on empirical criterions of outliers and can not deal with complicated situations. In our work, we both explain theoretical analyses of outliers and present an efficient outlier identification approach, by which we can achieve much better matching accuracy in complicated applications.

{\bf Interpretability for graph matching}.
The probability-based works~\cite{[2008-Zass-cvpr],[2013-Egozi]} formulated GM from the maximum-likelihood estimation perspective. A random walk view~\cite{[2010-Cho-eccv]} was introduced by simulating random walks with re-weighting jumps for GM. Some machine learning-based works~\cite{[2009-Caetano-pami],[2012-Leordeanu-ijcv]} went further to adjust attributes of graphs or improve the affinity matrix $\mathbf{K}$ (in Eq.~\eqref{eq:gm=lawler}) based on priors learned from real data. A functional representation framework~\cite{[2019-FRGM]} was proposed to give geometric insights for both general and Euclidean GM. The pioneering works~\cite{[2018-Zanfir],[2019_yjc]} presented an end-to-end deep learning framework for GM. Our work aims to establish the mathematical foundation for GM with outliers and enhance its theoretical rationality.

{\bf Computational efficiency}.
Some existing works aimed to reduce the costly space complexity caused by $\mathbf{K}$ in Eq.~\eqref{eq:gm=lawler}. A typical work was the factorized graph matching~\cite{[2016-Zhou-pami]}, which factorized $\mathbf{K}$ as Kronecker product of several smaller matrices.  However, it is highly time-consuming in practice due to the verbose iterations during optimization. Some methods like the graduated assignment method~\cite{[1996-Gold]} and the integer-projected fixed point algorithm~\cite{[2009-Leordeanu-nips]} proposed specific fast approximations while ended with unsatisfactory matching results. As comparison, our method has low space and time complexity and achieves better trade-off between time consumption and matching accuracy.

\section{Graph matching with outliers}\label{sec:bpgm}
This section revisits the general formulation of GM and presents the theoretical foundation for GM with outliers.

\subsection{General formulation of graph matching}
Given two attributed graphs  $\mathcal{G}=\{\mathcal{V},\mathcal{E}\},\mathcal{G}'=\{\mathcal{V}',\mathcal{E}'\}$, where $\mathcal{V}=\{V_i\}_{i=1}^m$ and $\mathcal{V}'=\{V'_a\}_{a=1}^n$ represent the node sets (assume $m\leq n$), $\mathcal{E}\subseteq \mathcal{V}\times \mathcal{V}$ and $\mathcal{E}'\subseteq \mathcal{V}'\times \mathcal{V}'$ denote the edge sets. Generally, for each graph, {\eg}, $\mathcal{G}$, the edges are represented by a (weighted) adjacency matrix $\mathcal{E}\in \mathbb{R}^{m\times m}$, where $\mathcal{E}_{ij} > 0$ if there is an edge $(V_{i},V_{j})$, and $\mathcal{E}_{ij}=0$ otherwise. In practice, graph $\mathcal{G}$ is usually associated with node  attribute $\mathbf{v}_{i}\in \mathbb{R}^{d_v}$ of node $V_i$ and edge attribute $\mathbf{A}_{ij}\in\mathbb{R}^{d_e}$ of edge $\mathcal{E}_{ij}$; the same to graph $\mathcal{G}'$.

Solving GM problem is to find an optimal binary correspondence $\mathbf{P}\in \left\{0,1\right\}^{m\times n}$ ,
where $\mathbf{P}_{ia}=1$ when the nodes $V_i\in\mathcal{V}$ and $V'_a\in\mathcal{V}'$ are matched, and $\mathbf{P}_{ia}=0$ otherwise. To find such an optimal correspondence, GM methods generally minimize or maximize an objective function that measures the mutual (dis-)similarity between graphs.

As a typical Quadratic Assignment Problem (QAP), GM formulated as Lawler's QAP~\cite{[1963-Lawler],[2005-Leordeanu],[2009-Leordeanu-nips],[2010-Cho-eccv],[2016-Zhou-pami]} has been favored to maximize the sum of node and edge similarities
{\small
	\begin{equation}
	\max_{\mathbf{P}\in\mathcal{P}}\mathbf{P}_v^{\text{T}}\mathbf{K}\mathbf{P}_v =
	\sum_{i,a}\mathbf{P}_{ia}\mathbf{K}_{ia;ia} + \sum_{{\tiny{\substack{(i,j),(a,b)}}}}\mathbf{P}_{ia}\mathbf{K}_{ij;ab}\mathbf{P}_{jb},
	\label{eq:gm=lawler}
	\end{equation}}where $\mathbf{P}_v$ is the columnwise vectorized replica of $\mathbf{P}$. The affinity matrix $\mathbf{K}\in\mathbb{R}^{mn\times mn}$ has diagonal element $\mathbf{K}_{ia;ia}$ measuring the node affinity calculated with node attributes $(\mathbf{v}_i,\mathbf{v}'_a)$ and non-diagonal element $\mathbf{K}_{ia;jb}$ measuring the edge affinity calculated with edge attributes $(\mathbf{A}_{ij},\mathbf{B}_{ab})$.

Another famous formulation is Koopmans-Beckmann's QAP~\cite{[1957-Koopmans],[1988-Umeyama],[1993-Almohamad],[2009-Zaslavskiy-pami]}, which maximizes a trace-form objective function measuring the node and edge similarities
{\begin{equation}
	\max_{\mathbf{P}\in\mathcal{P}} \text{tr}(\mathbf{U}^{\text{T}}\mathbf{P})+\lambda\text{tr}(\mathcal{E}\mathbf{P}\mathcal{E}'\mathbf{P}^{\text{T}}),
	\label{eq:gmKoom22}
	\end{equation}}where $\{\mathbf{U}_{ia}\}\in \mathbb{R}^{m\times n}$ measures the node similarity between $V_i$ and $V'_a$ and $\lambda\ge 0$ is a weight. 


Generally, GM methods impose the one-to-(at most)-one constraint, {\ie}, the feasible field $\mathcal{P}$ can be defined as
{\begin{equation}\label{eq:feasible1}
	\mathcal{P}\triangleq \left\{\mathbf{P}\in \left\{0,1\right\}^{m\times n}; \mathbf{P1}=\mathbf{1},\mathbf{P}^{\text{T}}\mathbf{1}\le \mathbf{1}\right\},
\end{equation}}where $\mathbf{1}$ is a columnwise unit vector. In fact, Eq.~\eqref{eq:feasible1} means that both inliers and outliers are equally treated to find their correspondences. Some methods like~\cite{[2009-Caetano-pami],[2013-Torresani-pami]} replace $\mathbf{P1}=\mathbf{1}$ by $\mathbf{P1}\leq\mathbf{1}$ to relax the one-to-(at most)-one constraint. However, they still lack of intrinsic theoretical analyses for the numerous outliers arising in both graphs.


\subsection{Zero-assignment constraint for outliers}

As stated previously in Sec.~\ref{sec:introduction}, we aim to only match inliers to inliers and suppress the matchings of outliers. To achieve our goal, we present the zero-assignment constraint for outliers in this section. Denoting the number of inliers in $\mathcal{G}'$ and $\mathcal{G}$ as $k$ $(0<k\leq m\leq n)$, for better understanding, we first introduce some basic definitions in the following.

\begin{defn}
	Denote $\mathscr{A}=\{1,2,...,m\}$ as the index set of nodes in graph $\mathcal{G}$. The index sets of inliers and outliers of $\mathcal{G}$ are respectively defined as,	
	\begin{align} 	
	\mathscr{A}_I &\triangleq \{i\in\{1,2,...,m\};\mathcal{V}_i \text{ is an inlier of } \mathcal{G} \},\\
	\mathscr{A}_O &\triangleq \{o\in\{1,2,...,m\};\mathcal{V}_o \text{ is an outlier of } \mathcal{G} \}.
	\end{align}
	The index sets $\mathscr{B}=\{1,2,...,n\}$, $\mathscr{B}_I$ and $\mathscr{B}_O$ are similarly defined for graph $\mathcal{G}'$. Obviously, we have $|\mathscr{A}_I|=|\mathscr{B}_I|=k$. The inliers and outliers sets are complementary and disjoint.
	\label{def:inl_out}
\end{defn}
\begin{prop}\label{prop:prop1}
	\begin{align}
	&\mathscr{A}_I\cup\mathscr{A}_O=\mathscr{A},\quad\mathscr{A}_I\cap\mathscr{A}_O=\varnothing,\\
	&\mathscr{B}_I\cup\mathscr{B}_O=\mathscr{B},\quad\mathscr{B}_I\cap\mathscr{B}_O=\varnothing.
	\end{align}
	where $\varnothing$ denotes the empty set.
\end{prop}

Next, we derive the zero-assignment constraint for outliers. Mathematically, the matching between $\mathcal{G}$ and $\mathcal{G}'$ consisting of inliers and outliers can be defined by a partial permutation $\mathbf{\tau}$ and a partial permutation matrix $\mathbf{P}$ as follows.
\begin{defn} The {\em partial permutation} $\tau$ between $\mathcal{G}$ and $\mathcal{G}'$ is defined as $\tau : \mathscr{A} \to \mathscr{B}$,
	\begin{equation}
		 i\mapsto a=\tau(i)\in\mathscr{B}_I \text{ if } i\in\mathscr{A}_I; a=\varnothing \text{ if } i\in\mathscr{A}_O.		 
	\end{equation}
	 And the inverse of $\tau$ can also be defined as $\tau^{-1} : \mathscr{B} \to \mathscr{A}$,
	 \begin{equation}
	  a\mapsto i=\tau^{-1}(a)\in\mathscr{A}_I \text{ if } a\in\mathscr{B}_I; i=\varnothing \text{ if } a\in\mathscr{B}_O.
	 \end{equation}\label{def:permu1}
\end{defn}

Given $\tau$, the matching (or correspondence) between $\mathcal{G}$ and $\mathcal{G}'$ can be equivalently expressed by the {\em partial permutation matrix} $\mathbf{P}\in\{0,1\}^{m\times n}$ compatible with $\tau$ as

\begin{defn} For $\mathbf{P}\in\{0,1\}^{m\times n}$ compatible with $\tau$,
	\begin{itemize}
		\item[-] {\em One-to-one constraint for inliers:} $\forall i\in\mathscr{A}_I$,
	\begin{equation}
	\mathbf{P}_{i,a=\tau(i)}=1,\mathbf{P}_{i,a\neq\tau(i)}=0, a\in\mathscr{B}_I.
	\end{equation}
	\item[-] {\em Zero-assignment constraint for outliers:} 
	\begin{equation}
	\mathbf{P}_{i,:}\equiv\mathbf{0}^{\text{T}}, \forall i\in\mathscr{A}_O \text{ and } \mathbf{P}_{:,a}\equiv\mathbf{0}, \forall a\in\mathscr{B}_O.
	\end{equation}
	\end{itemize}
where $\mathbf{P}_{i,:}$ (or $\mathbf{P}_{:,a}$) is a row (or column) vector of $\mathbf{P}$, and $\mathbf{0}$ is a columnwise zero vector.	
\label{def:permu2}
\end{defn}

By this means, the feasible filed $\mathcal{P}_k$ can be redefined as
\begin{equation}
\left\{\mathbf{P}\in \{0,1\}^{m\times n}; \mathbf{P1}\leq\mathbf{1},\mathbf{P}^{\text{T}}\mathbf{1}\le \mathbf{1},\mathbf{1}^{\text{T}}\mathbf{P}\mathbf{1}=k\right\}.
\end{equation}	
The explicit equation constraint $\mathbf{1}^{\text{T}}\mathbf{P}\mathbf{1}=k$ will be used to both present a proof for the rationality of our proposed objective function in Sec.~\ref{sec:obj_condition} and design an efficient optimization algorithm in Sec.~\ref{sec:optimiza}.

\subsection{Consistency and distinguishability}
Empirically, the GM methods assume that the unary and pairwise attributes of inlier $i\in\mathscr{A}_I$ and edge $(i,j)\in\mathscr{A}_I\times \mathscr{A}_I$ are consistent with those of the ideal matchings $a\in\mathscr{B}_I$ and $(a,b)\in\mathscr{B}_I\times \mathscr{B}_I$, while the outliers are on the contrary. Based on this empirical criterion, we furthermore elaborate a quantitative consistency of inliers and distinguishability between inliers and outiers, on which bases the rationality of our objective function can be guaranteed. 

Denote $\{\mathbf{D}_{ia}\}_{ia}$ as the dissimilarity between nodes $V_i\in\mathcal{V}$ and $V_a'\in\mathcal{V}'$, $\{\mathbf{A}_{ij}\}_{ij}$ and $\{\mathbf{B}_{ab}\}_{ab}$ are the edge attributes of edges $(V_i,V_j)\in\mathcal{E}$ and $(V'_a,V'_b)\in\mathcal{E}'$. Meanwhile, denote $\{\tau^*,\mathbf{P}^*\in\mathcal{P}_k\}$ as the ideal matching between $\mathcal{G}$ and $\mathcal{G}'$. Consequently, beyond the empirical criterion, we can induce the consistency of inliers and distinguishability between inliers and outliers by $\{\tau^*,\mathbf{P}^*\in\mathcal{P}_k\}$ as follows.

\begin{prop} Consistency between inliers.\label{prop:con_dis}
	\begin{itemize}\label{prop:consistency}
	\item[-]{Unary consistency}: 
	$\forall i\in\mathscr{A}_I,\forall a\in\mathscr{B}_I$,
		{\begin{align}
		&\mathbf{D}_{ia'}=\min\{\mathbf{D}_{ia},a\in\mathscr{B}\} \Leftrightarrow a'=\tau^*(i),\\
		&\mathbf{D}_{i'a}=\min\{\mathbf{D}_{ia},i\in\mathscr{A}\} \Leftrightarrow i'=\tau^{*-1}(a).
		\end{align}}
	\item[-]{Pairwise consistency}: 
	$\forall i,j \in\mathscr{A}_I,\forall a,b \in\mathscr{B}_I$,
		{\begin{align}
			||\mathbf{A}_{ij}-\mathbf{B}_{a'b'}|| &= \min\{||\mathbf{A}_{ij}-\mathbf{B}_{ab}||,a,b\in\mathscr{B}\}\nonumber\\
			&\Leftrightarrow a'=\tau^*(i),b'=\tau^*(j),\\
			||\mathbf{B}_{ab}-\mathbf{A}_{i'j'}|| &= \min\{||\mathbf{B}_{ab}-\mathbf{A}_{ij}||,i,j\in\mathscr{A}\}\nonumber\\
			&\Leftrightarrow i'=\tau^{*-1}(a),j'=\tau^{*-1}(b).
		\end{align}}
	\end{itemize}
\end{prop}
\begin{prop} Distinguishability between inliers and outliers.
	\begin{itemize}\label{prop:dis}
	\item[-]{Unary distinguishability}:
	$\forall (i,a)\in\mathscr{A}_O\times\mathscr{B}$ or $\mathscr{A}\times\mathscr{B}_O$,
	\begin{align}
		\mathbf{D}_{ia}\geq \max \{\mathbf{D}_{i'\tau^*(i')},i'\in\mathcal{A}_I\}.
	\end{align}
	\item[-]{Pairwise distinguishability}:
	$\forall (i,a),(j,b) \in\mathscr{A}_O\times \mathscr{B}$ or  $\mathscr{A}\times \mathscr{B}_O$,
	\begin{align}
	||\mathbf{A}_{ij}-\mathbf{B}_{ab}|| 
	&\geq \max\{||\mathbf{A}_{i'j'}-\mathbf{B}_{\tau^*(i')\tau^*(j')}||\nonumber\\
	&\qquad\qquad,i',j'\in\mathcal{A}_I\},\\
	||\mathbf{B}_{ab}-\mathbf{A}_{ij}|| 
	&\geq \max\{||\mathbf{B}_{a'b'}-\mathbf{A}_{\tau^{*-1}(a')\tau^{*-1}(b')}||\nonumber\\
	&\qquad\qquad,a',b'\in\mathcal{B}_I\}.
	\end{align}
	\end{itemize}
where $||\cdot||$ is an Euclidean norm.
\end{prop}
By this means, we present a quantitative mathematical criteria of the local characteristics and mutual relationships of inliers and outliers, which is more concise and clear than empirical criteria. More importantly, the propositions above inspires us how to construct a reasonable objective function and find out a sufficient condition for proving the rationality.

\subsection{Objective function with sufficient condition}\label{sec:obj_condition}
A reasonable objective function $F(\mathbf{P})$ should satisfy two main properties: (1) preserve the unary and pairwise consistencies between the matched nodes (or edges) of two graphs and (2) achieve its optimum only at the ideal matching $\mathbf{P}^*$. Overall, our objective function is defined as
\begin{equation}\label{eq:obj}
	\min_{\mathbf{P}\in\mathcal{P}_k} F(\mathbf{P})=\lambda_1 F_u(\mathbf{P})+\lambda_2F_p(\mathbf{P}),	
\end{equation}where $F_u(\mathbf{P})$ and $F_p(\mathbf{P})$ are the unary and pairwise potentials. Precisely, we set $F_u(\mathbf{P})=\sum\limits_{ia}\mathbf{D}_{ia}\mathbf{P}_{ia}$ and
\begin{align}
	F_p(\mathbf{P})&\triangleq F_{p_1}(\mathbf{P}) + F_{p_2}(\mathbf{P})\\ &\triangleq\sum_{ij}\mathcal{E}_{ij}||\mathbf{A}_{ij}-\sum_{a,b}\mathbf{P}_{ia}\mathbf{B}_{ab}\mathbf{P}_{jb} ||^2 \nonumber\\
	&\quad +\sum_{ab}\mathcal{E}'_{ab}||\mathbf{B}_{ab}-\sum_{i,j}\mathbf{P}_{ia}\mathbf{A}_{ij}\mathbf{P}_{jb}||^2\\
	&\triangleq||\mathbf{A}-\mathbf{PB}\mathbf{P}^{\text{T}} ||_\mathcal{E}^2 + ||\mathbf{B}-\mathbf{P}^{\text{T}}\mathbf{AP}||_{\mathcal{E}'}^2.
\end{align}

The property (1) is guaranteed since the minimization of $F(\mathbf{P})$ tends to find the minimizer $\hat{\mathbf{P}}$ that matches the nodes and edges in $\mathcal{G}$ (or $\mathcal{G}'$) to the mostly-consistent nodes and edges in $\mathcal{G}'$ (or $\mathcal{G}$). Next, we should make sure that it also satisfies the property (2). However, due to the cluttered outliers arising in both graphs, it may not hold for any arbitrarily given weighted adjacency matrices $\mathcal{E},\mathcal{E}'$ or edge attributes $\mathbf{A},\mathbf{B}$. Furthermore, we put forward a sufficient condition to support it. 

\begin{prop}Sufficient condition for objective function. Assume that the weighted adjacency matrices $\mathcal{E},\mathcal{E}'$ and  edge attributes $\mathbf{A},\mathbf{B}$ satisfy that
	\begin{align}
		\mathcal{E}_{i\in\mathscr{A}_I,j\in\mathscr{A}_I} 
		&\geq \mathcal{E}_{i\in\mathscr{A},j\in\mathscr{A}_O}, \mathcal{E}_{i\in\mathscr{A}_O,j\in\mathscr{A}}\label{eq:prop3_1},\\
		||\mathbf{A}_{i\in\mathscr{A}_I,j\in\mathscr{A}_I}||
		&\geq ||\mathbf{A}_{i\in\mathscr{A},j\in\mathscr{A}_O}||, ||\mathbf{A}_{i\in\mathscr{A}_O,j\in\mathscr{A}}||,\label{eq:prop3_2}
	\end{align}
	and the same to $\mathcal{E}'$ and $\mathbf{B}$. Then, it's sufficient to prove that 
	\begin{equation}
		\forall \mathbf{P}\in\mathcal{P}_k, ~F(\mathbf{P})\geq F(\mathbf{P}^*),\label{eq:rational}
	\end{equation}the equation holds if and only if~~$\mathbf{P}=\mathbf{P}^*$.
	\label{prop:prop_condition}
\end{prop}

\begin{proof}
 Due to the over-length of the entire proof, we give the details in our supplementary materials, which also demonstrate the intrinsic differences between GM on simple graphs and on complicated graphs.
\end{proof}

Note that, the Eq.~\eqref{eq:prop3_1} and~\eqref{eq:prop3_2} tell us how to calculate proper $\{\mathcal{E}_{ij}\}$, $\{\mathbf{A}_{ij}\}$ (or $\{\mathcal{E}'_{ab}\}$, $\{\mathbf{B}_{ab}\}$): we should compute $\mathcal{E}_{ij}$ and $\mathbf{A}_{ij}$ to measure the similarities between the two end-nodes in edge $(i,j)$ such that edges linked by two inliers have higher similarities than the edges linked by inlier-outlier or outlier-outlier. It will be followed and validated in the experiments section Sec.~\ref{sec:exper}.

\section{Outlier-robust graph matching algorithm}\label{sec:numerical}

In this section, we propose an efficient algorithm to solve Eq.~\eqref{eq:obj} and then design an outlier identification approach.

\subsection{Optimization algorithm}\label{sec:optimiza}
Our optimization algorithm is based on the Frank-Wolfe method~\cite{[2015-Simon-nips],[2016-Lafond]}, which is widely used for convex or non-convex optimization and achieve at least sub-linear convergence rate. Since it is a continuous line-search-based method, we should relax the discrete ${\mathcal{P}}_k$ into the continuous $\hat{\mathcal{P}}_k$ by relaxing $\mathbf{P}_{ia}\in\{0,1\}$ into $\mathbf{P}_{ia}\in[0,1]$. Given $F(\mathbf{P})$ is differentiable and $\hat{\mathcal{P}}_k$ is convex, the Frank-Wolfe method iterates the following steps till it converges:
{\begin{align}
	\tilde{\mathbf{P}}^{(t+1)}&\in \mathop{\text{argmin}}\limits_{{\mathbf{P}}\in \hat{\mathcal{P}}_k}\triangleq\langle\nabla F(\mathbf{P}^{(t)}),\mathbf{P}\rangle,\label{eq:fw1}\\
	\mathbf{P}^{(t+1)}&=\mathbf{P}^{(t)} + \alpha^{(t)}(\tilde{\mathbf{P}}^{(t+1)}-\mathbf{P}^{(t)}),\label{eq:fw2}
\end{align}}where $\nabla F(\mathbf{P}^{(t)})$ is the gradient of $F(\mathbf{P})$ at $\mathbf{P}^{(t)}$ and $\alpha^{(t)}$ is the step size obtained by exact or inexact line search~\cite{[1965-Goldstein]}. 

{\bf Gradient computation.} The gradient $\nabla F(\mathbf{P})$ can be efficiently calculated by matrix operations as follows,
\begin{align}
	&\mathbf{W}_1\triangleq 4||\mathbf{PBP}^{\text{T}}-\mathbf{A}||\otimes \text{sign}(\mathbf{PBP}^{\text{T}}-\mathbf{A})\otimes\mathcal{E},\\
	&\mathbf{W}_2\triangleq 4||\mathbf{P}^{\text{T}}\mathbf{AP}-\mathbf{B}||\otimes \text{sign}(\mathbf{P}^{\text{T}}\mathbf{AP}-\mathbf{B})\otimes\mathcal{E}',\\
	&\nabla F(\mathbf{P})=\lambda_1 \mathbf{D} + \lambda_2 [\mathbf{W}_1\mathbf{PB}^{\text{T}}+\mathbf{AP}\mathbf{W}_2^{\text{T}}],
\end{align}where $\otimes$ is the pointwise multiplication and $\text{sign}(\cdot)$ is the sign function.

{\bf The k-cardinality LAP.} Eq.~\eqref{eq:fw1} plays a key role of the optimization. It is a linear programming (LP) problem that can be solved by LP algorithms like interior point method~\cite{[1994-Nesterov-siam]}. However, such methods have costly time complexity $O(m^3n^3/\text{ln(mn)})$~\cite{[1999-Anstreicher-siam]}. Fortunately, one can prove that $\tilde{\mathbf{P}}^{(t+1)}$ is an extreme point~\cite{[1995-extreme]} of $\hat{\mathcal{P}}_k$, thus, $\tilde{\mathbf{P}}^{(t+1)}\in\mathcal{P}_k$. Therefore, Eq.~\eqref{eq:fw1} boils down to a k-cardinality linear assignment problem (kLAP)~\cite{[1997-DellAmico]}. We can adopt the approach~\cite{[2004-Volgenant]} by which the kLAP is transformed into a standard LAP that can be efficiently solved by the Hungarian~\cite{[2010-Kuhn]} or LAPJV~\cite{[1987-Jonker]} algorithm with much less time complexity $O(n^3)$.

{\bf Regularization.} Someone may doubt that the explicit equation constraint $\mathbf{1}^{\text{T}}\mathbf{P}\mathbf{1}=k$ in the feasible filed is too strong. We can replace it with an implicit regularization term $(\mathbf{1}^{\text{T}}\mathbf{P}\mathbf{1}-k)^2$ and obtain a new objective function as
\begin{equation}\label{eq:obj_reg}
\min_{\mathbf{P}} F_{r}(\mathbf{P},k)= F(\mathbf{P})+\lambda_0(\mathbf{1}^{\text{T}}\mathbf{P}\mathbf{1}-k)^2.
\end{equation}
We set $\lambda_0=1$ is this paper. To solve Eq.~\eqref{eq:obj_reg}, we can adopt the alternating optimization strategy: alternatively find the minimizer $\hat{\mathbf{P}}$ of Eq.~\eqref{eq:obj_reg} by Frank-Wolfe method with fixed $k$ and then update $k=\mathbf{1}^{\text{T}}\hat{\mathbf{P}}\mathbf{1}$. Note that, in this case, Eq.~\eqref{eq:fw1} is solved by LP algorithms (interior point method in this paper) rather than the kLAP solvers since the constraint $\mathbf{1}^{\text{T}}\mathbf{P}\mathbf{1}=k$ dose not hold during solving Eq.~\eqref{eq:fw1}.

{\bf Computational complexity.} Since we do not use the affinity matrix $\mathbf{K}$, the space complexity is only $O(n^2)$. In optimization, each iteration takes time complexity $O(n^3)$ to solve the k-LAP or $O(m^3n^3/\text{ln(mn)})$ to solve the LP, and $O(m^2n+mn^2)$ to compute the values and gradients of objective function. We are advised to adopt the kLAP-based approach based on the experimental analyses in Sec.~\ref{sec:exper}.

\subsection{Outlier identification and removal}

After minimizing $F(\mathbf{P})$ or $F_r(\mathbf{P},k)$, we obtain an optimal correspondence matrix $\hat{\mathbf{P}}$ that has two advantages beneficial to outlier identification: (1) $\hat{\mathbf{P}}$ optimally preserves the structural alignments between the two matched graphs. (2) The nearly zero-valued vectors $\hat{\mathbf{P}}_{i,:}\approx \mathbf{0}^{\text{T}}$ or $\hat{\mathbf{P}}_{:,a}\approx \mathbf{0}$ indicate that the node $V_i\in\mathcal{G}$ or $V_a'\in\mathcal{G}'$ can be identified as outliers, as an example shown in Fig.~\ref{fig:out_ide} (a).

An outlier removal approach is proposed based on this outlier identification criterion. Given $\hat{\mathbf{P}}$, we first calculate two vectors as $\hat{\mathbf{P}}\mathbf{1}=\{{\hat{\mathbf{P}}_{i,:}\mathbf{1}}\}_{i=1}^m$ and $\mathbf{1}^{\text{T}}\hat{\mathbf{P}}=\{\mathbf{1}^{\text{T}}\hat{\mathbf{P}}_{:,a}\}_{a=1}^n$, whose components with smaller values are more likely to be outliers. Then, $\hat{\mathbf{P}}\mathbf{1}$ and $\mathbf{1}^{\text{T}}\hat{\mathbf{P}}$ form the 2-dimensional coordinates of coupled nodes $\{(V_i, V_a')\}_{i,a}$ in the joint probability space, where the inliers and outliers can be significantly separated and clustered ({\eg}, by k-means) into two classes, see an example in Fig.~\ref{fig:out_ide} (b). Assume that $m',n'$ nodes of the two graphs are clustered as inliers by the clustering step, if $m'<k$ or $n'<k$, we pick out $k-m'$ or $k-n'$ nodes left with higher component values and put them back into inliers. If $m'>k$ or $n'>k$, the nodes with component values less than $0.5$ will also be chosen as outliers. We iteratively execute this outlier removal procedure and then refine the inliers of two graphs till the enumerations of inliers keep unchanged. At last, the optimal solution solved w.r.t the refined graphs is our final matching result.

\begin{figure}
	\centering
	{\includegraphics[width=0.48\linewidth]{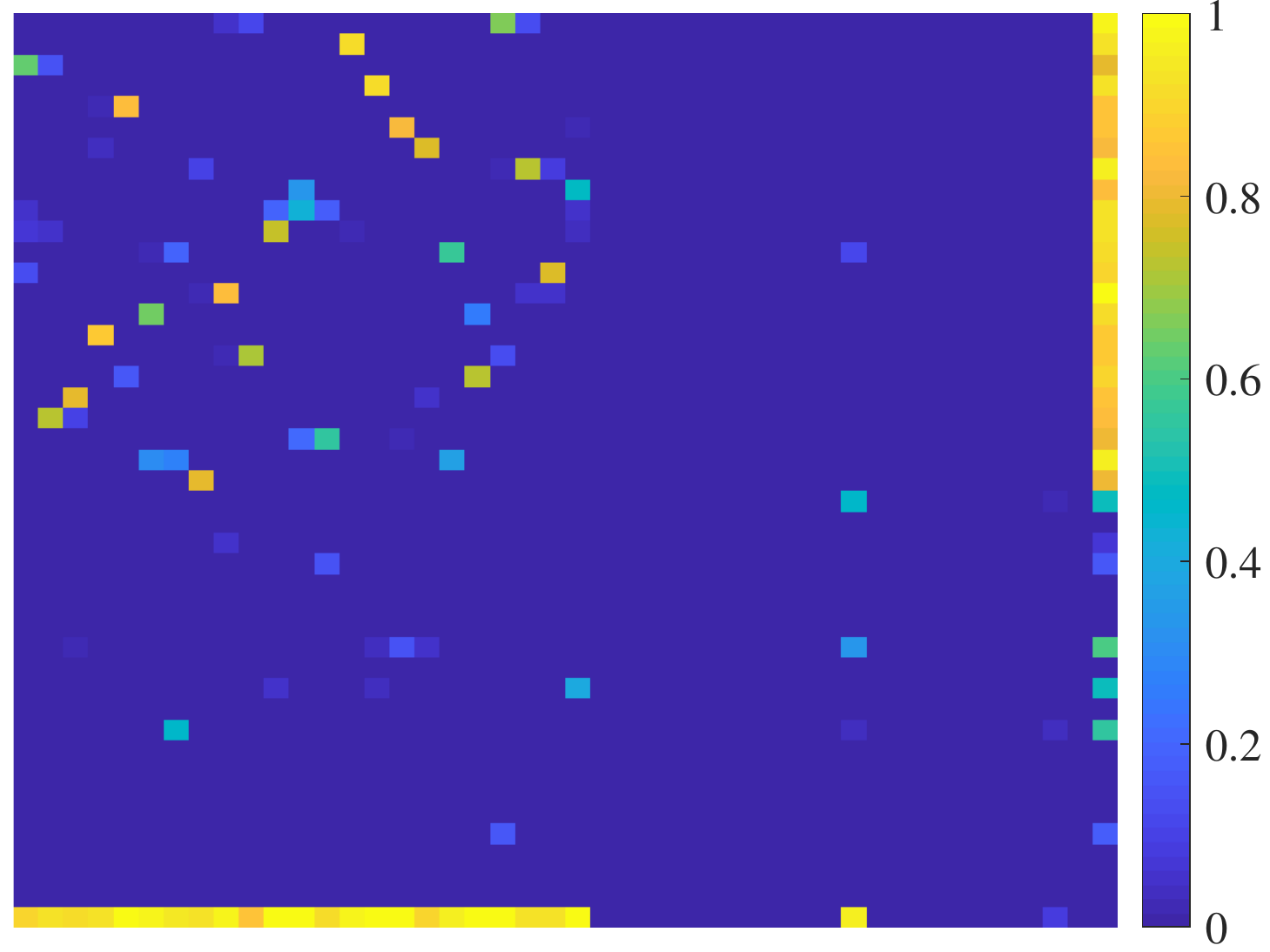}}
	\hspace{2mm}{\includegraphics[width=0.48\linewidth]{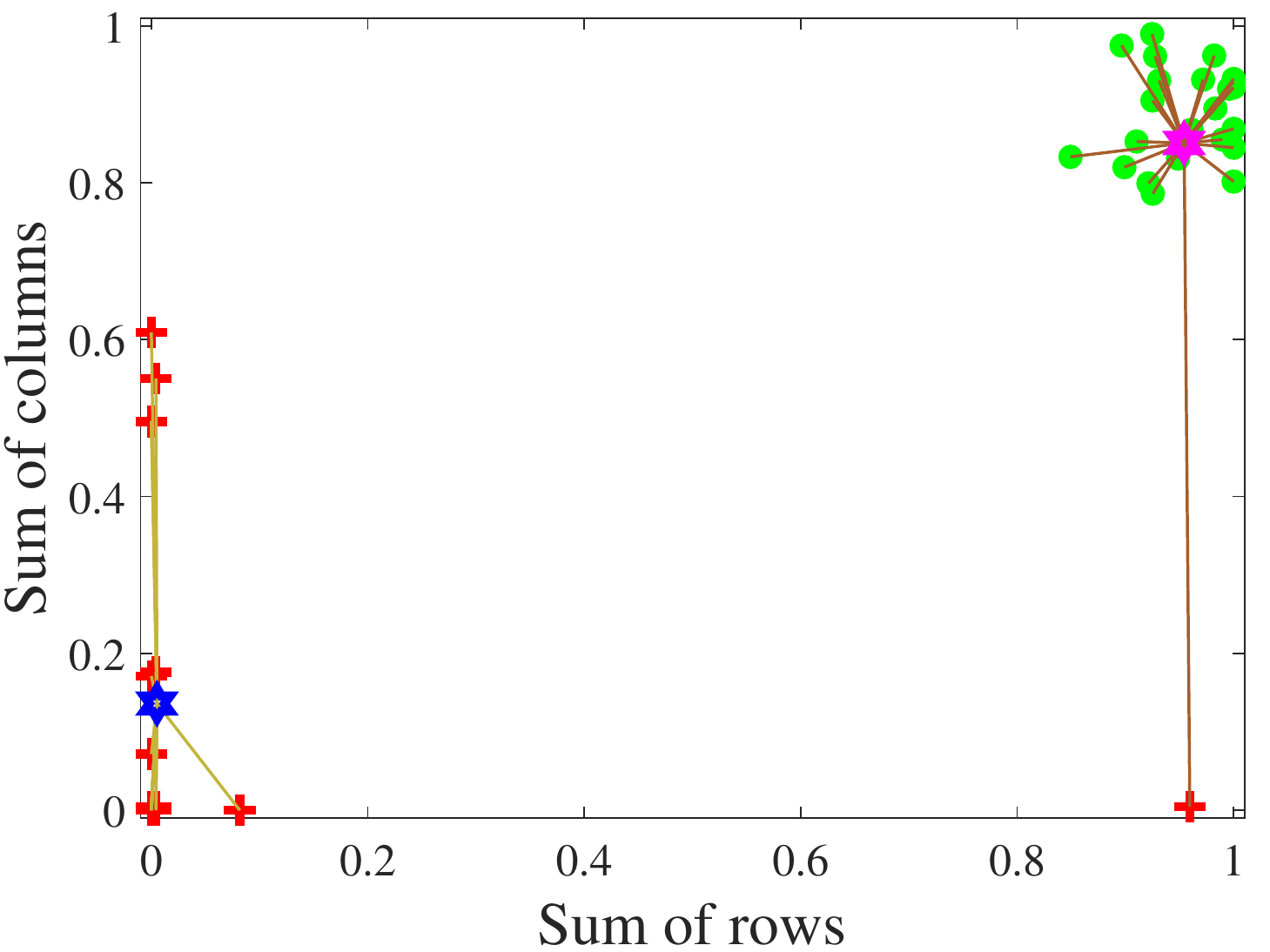}}
	\vspace{-5mm}
	\caption{An example of the outlier identification and removal w.r.t. Fig.~\ref{fig:fig1_com} (b). {\bf Left}: the last row and column show the sums of column and row vectors of the optimal correspondence matrix $\hat{\mathbf{P}}$. {\bf Right}: the inliers (green dots) and outliers (red signs) can be significantly separated and clustered into two classes. Note that, there is one outlier in each graph clustered as inlier due to its high similarity with the other inliers (see the matching result in Fig.~\ref{fig:fig1_com} (b)).}
	\label{fig:out_ide}
\end{figure}


\section{Experimental analysis}\label{sec:exper}

In this section, we evaluate and compare our methods (denoted as ZAC w.r.t. Eq.~\eqref{eq:obj} and ZACR w.r.t Eq.~\eqref{eq:obj_reg}) with state-of-the-art graph matching methods including GA~\cite{[1996-Gold]}, RRWM~\cite{[2010-Cho-eccv]}, MPM~\cite{[2014-Cho-cvpr]}, FGMD~\cite{[2016-Zhou-pami]}, BPFG~\cite{[2018-Wang]} and FRGM~\cite{[2019-FRGM]} on widely used complicated datasets in terms of matching accuracy and time consumption. The codes of the compared methods are downloaded from the author’s websites. Our code is available at \url{https://github.com/wangfudong/ZAC_GM}. For better evaluation of graph matching in the presence of outliers, we compute the commonly used indicators called {\itshape recall $=\frac{\#\{\text{correct matching}\}}{\#\{\text{groundtruth matching}\}}$}, {\itshape precision $=\frac{\#\{\text{correct matching}\}}{\#\{\text{total matching}\}}$} and {\itshape F-measure$=\text{\small{2}}\frac{\text{recall}\cdot\text{precision} }{\text{recall}+\text{precision}}$}. 

\subsection{Results on PASCAL dataset}

We first conducted experiments on graphs in PASCAL dataset~\cite{[2012-Leordeanu-ijcv]}, which consists of 30 and 20 pairs of car and motorbike images ({\eg}, Fig.~\ref{fig:fig1_com}), respectively. Each pair contains both inliers with known correspondence and randomly marked dozens of outliers. To generate graphs with outliers, we randomly selected 0, 4, ..., 20 outliers to both graphs, respectively. To generate the edges, our methods and FRGM applied complete graphs, while the others connected edges by Delaunay Triangulation, on which they achieved better performance than on complete graphs.

Similar with~\cite{[2016-Zhou-pami],[2019-FRGM]}, we set $\mathbf{K}_{ia;ia}=\text{exp}(-d(\mathbf{v}_i-\mathbf{v}_a'))$,
and $\mathbf{K}_{ia;jb}=\text{exp}(-\frac{1}{2}({|\mathbf{E}_{ij}-\mathbf{E}'_{ab}|}+{|\Theta_{ij}-\Theta'_{ab}|}))$, where $\mathbf{v}_i,\mathbf{v}_a'$ were shape context~\cite{[2002-Belongie-pami]}, $d(\mathbf{v}_i-\mathbf{v}_a')$ was the cost computed as $\chi^2$ test statistic~\cite{[2002-Belongie-pami]}, $\mathbf{E}_{ij},\mathbf{E}_{ab}'$ were distance matrices between nodes, $\Theta_{ij}, \Theta'_{ab}$ were the angles between the edges and the horizontal line. 
For our methods, we calculated $\mathbf{D}_{ia}=d(\mathbf{v}_i-\mathbf{v}_a')$ to measure the node dissimilarity. For the weighted adjacency matrices $\mathcal{E},\mathcal{E}'$ and edge attributes $\mathbf{A}_{ij},\mathbf{B}_{ab}$, in order to honor the proposition~\ref{prop:prop_condition}, we set $\mathcal{E}=1\oslash\mathbf{E},\mathcal{E}'=1\oslash\mathbf{E}'$ and $\mathbf{A}=\exp(-\mathbf{E}^2/\sigma_1^2),\mathbf{B}=\exp(-\mathbf{E}'^2/\sigma_2^2)$ with $\sigma_1,\sigma_2$ were the standard deviations of $\mathbf{E},\mathbf{E}'$. The weights in Eq.~\eqref{eq:obj} were $\lambda_1=\lambda_2=1$. 

\begin{figure}
	\centering
	\subfigure[Car]
	{\includegraphics[width=0.46\linewidth]{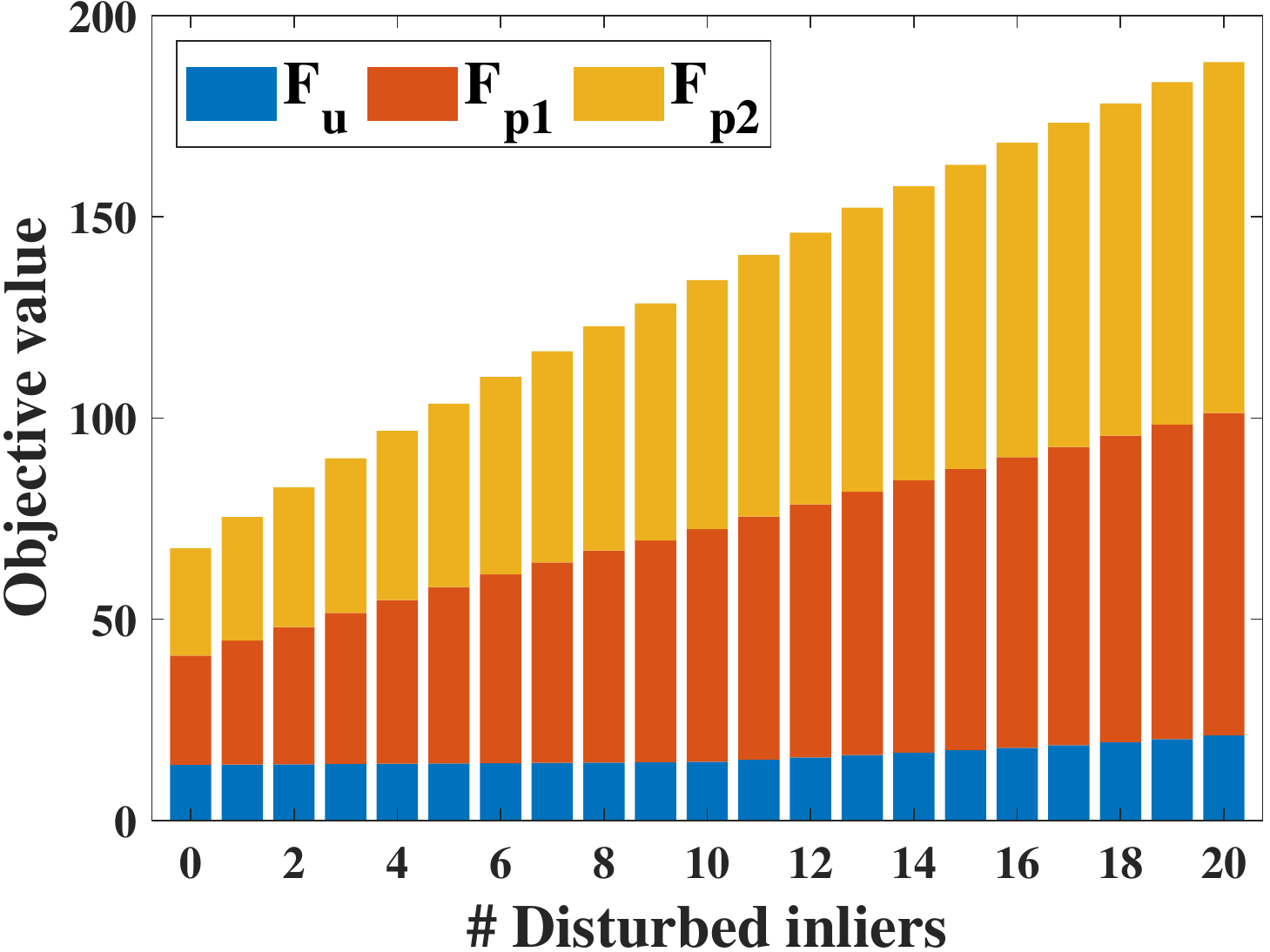}}
	\hspace{3mm}\subfigure[Motorbike]
	{\includegraphics[width=0.46\linewidth]{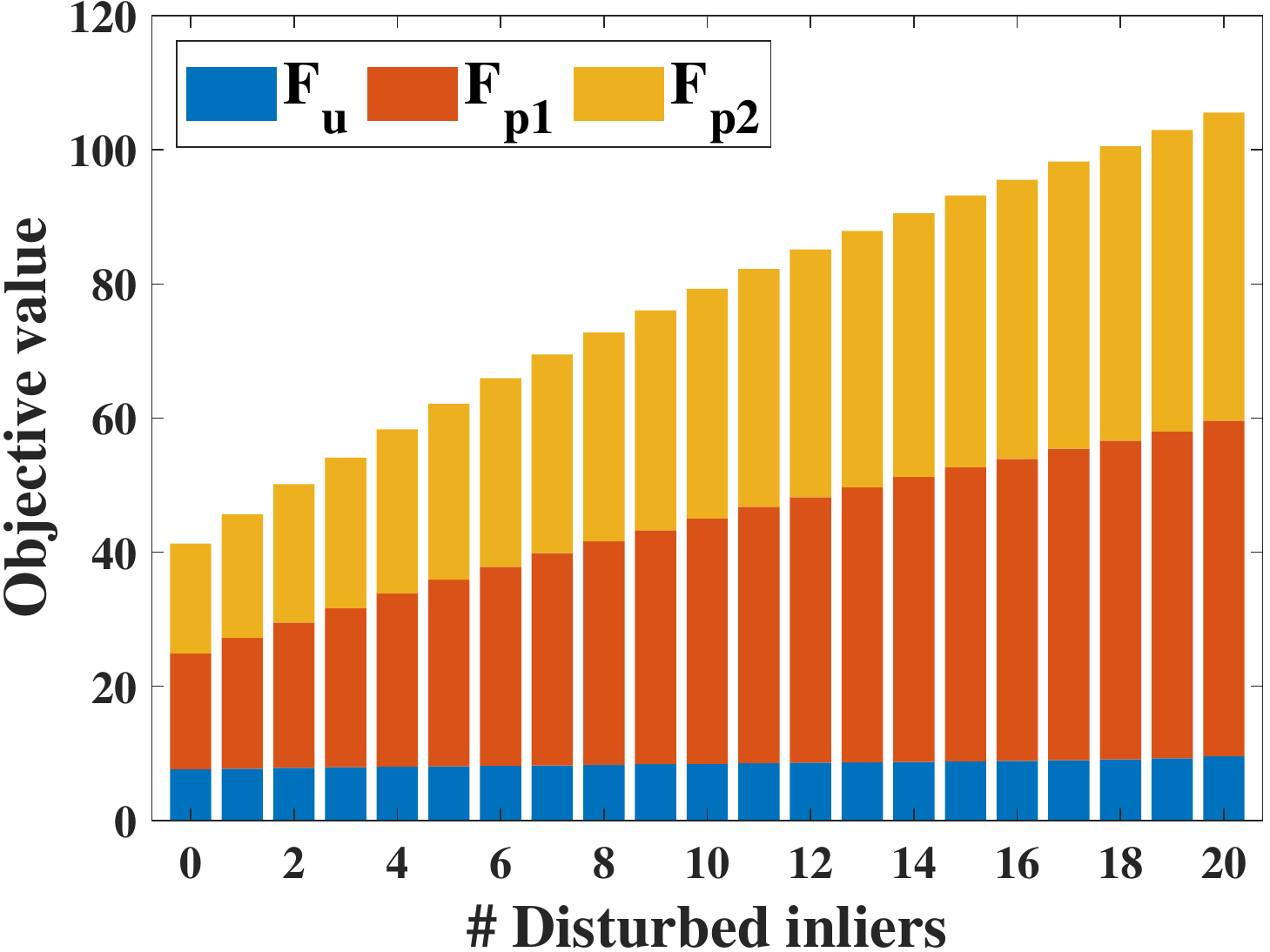}}
	\vspace{-4mm}
	\caption{Statistical verification of the minimum values of our objective function Eq.~\eqref{eq:obj}.}
	\label{fig:pascal_obj_val}
\end{figure}

First, we presented a statistical verification for proposition~\ref{prop:prop_condition}. For each graph pair with outliers, we randomly disturbed the ideal correspondences between inliers by forcing $0,1,...,20$ inliers to be incorrectly matched. Then, we applied our optimization algorithm to minimize the objective function Eq.~\eqref{eq:obj} under the mismatching constraints. We reported the series of obtained minimum values of objective function in Fig.~\ref{fig:pascal_obj_val}. It shows that, with increasing number of disturbed matchings of inliers, the minimum values of objective function become higher. Only with no mismatchings ({\ie}, the ideal ground-truth ${\mathbf{P}}^*$), the objective function achieves the lower limit of the series of minimum values. Namely, the proposition~\ref{prop:prop_condition} can be guaranteed with our settings and optimization algorithm in practical cases.

\begin{figure}
	\centering
	{\includegraphics[width=0.7\linewidth]{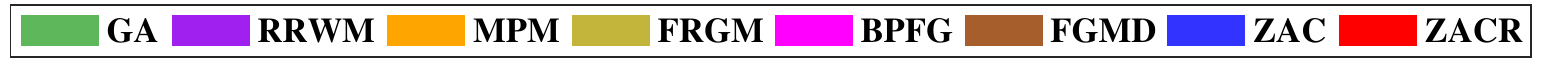}}
	
	\subfigure[Car]
	{\includegraphics[width=0.48\linewidth]{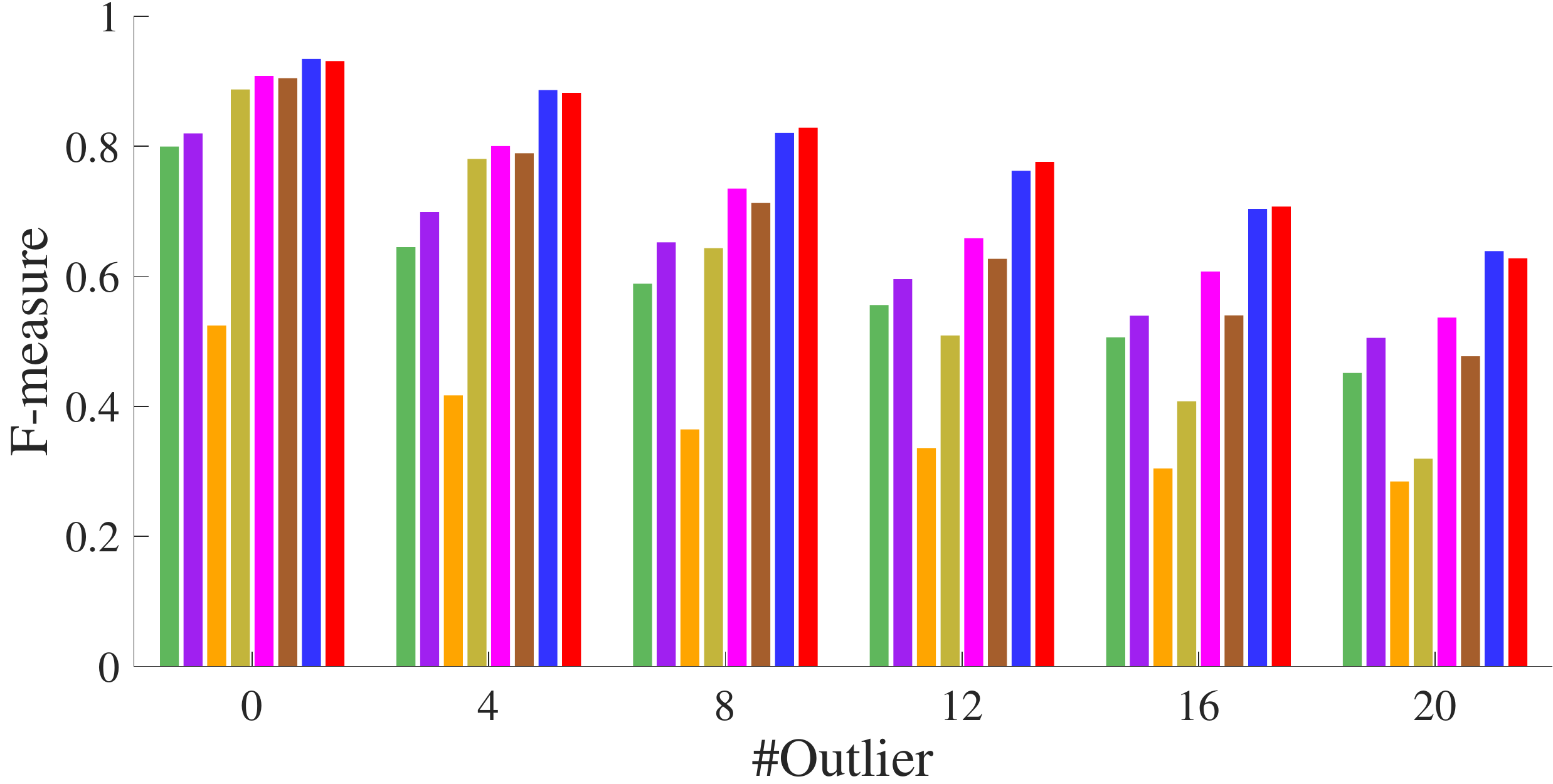}}
	\hspace{2mm}\subfigure[Motorbike]
	{\includegraphics[width=0.48\linewidth]{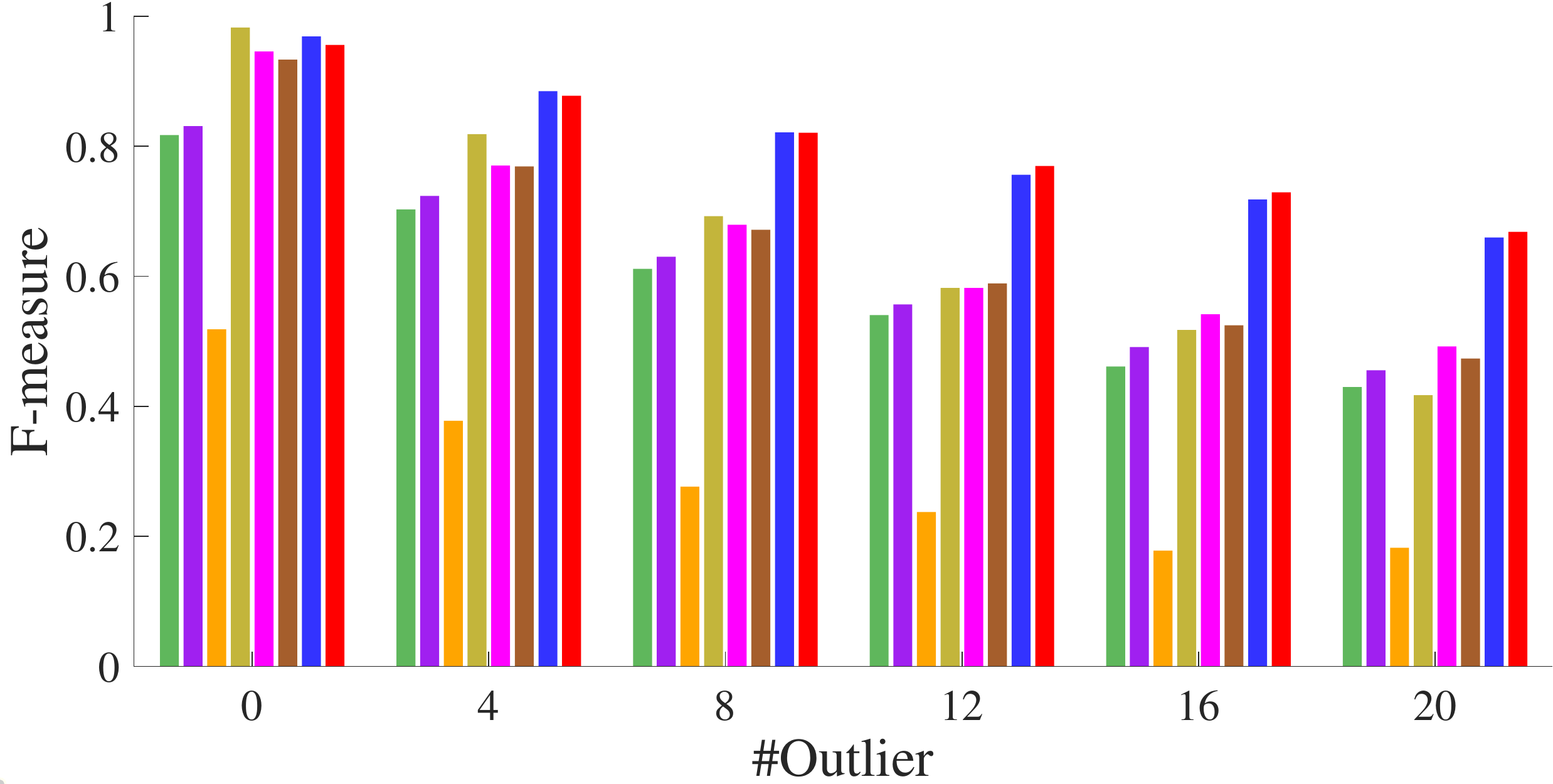}} 
	\vspace{-4mm}
	\caption{Average F-measure (\%) w.r.t. number of outliers.}
	\label{fig:pascal_fm_time}
\end{figure}

\begin{figure}
	\centering
	{\includegraphics[width=0.8\linewidth]{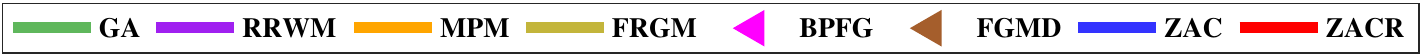}}
	\subfigure[Car]
	{\includegraphics[width=0.48\linewidth]{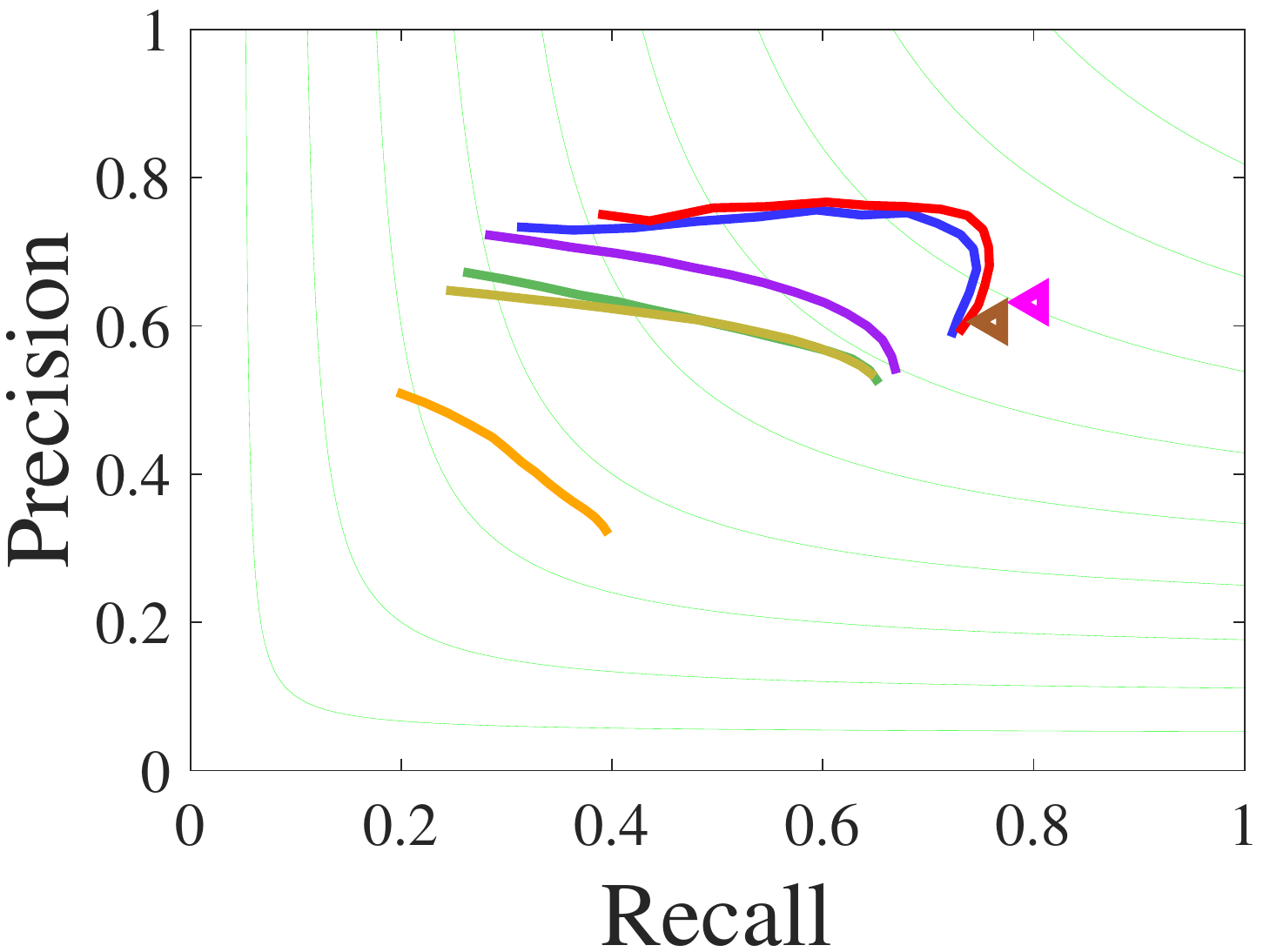}}
	\hspace{2mm}\subfigure[Motorbike]
	{\includegraphics[width=0.48\linewidth]{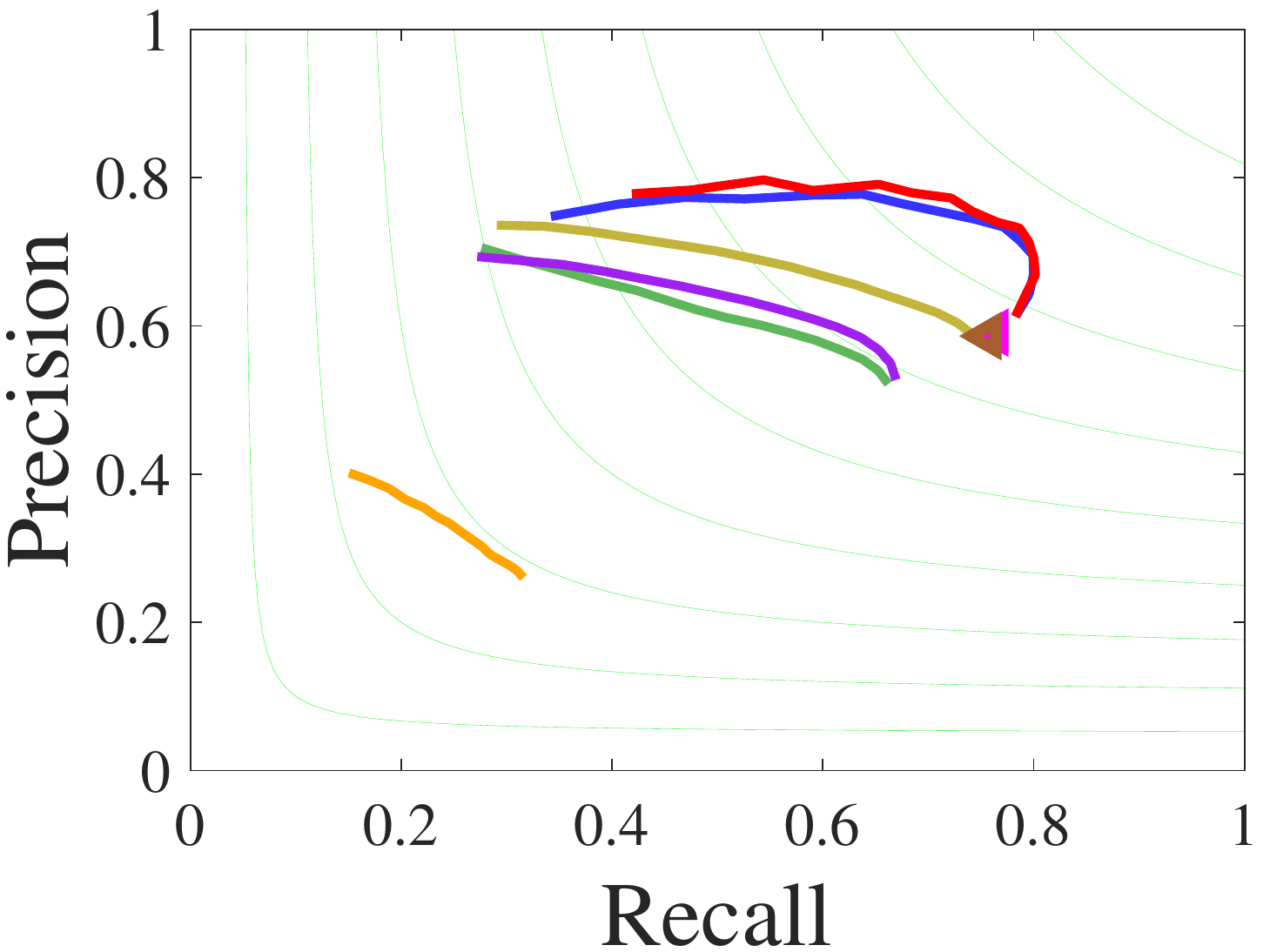}}	
	\vspace{-4mm}
	\caption{The average recall (\%) and precision (\%) w.r.t varying $ratio=0.3,0.35,...,1$ on PASCAL dataset.}
	\label{fig:pascal_pr}
\end{figure}

\begin{table}[!htb]
	\centering
	\scriptsize
	\begin{tabular}{m{0pt}p{25pt}|cccccc}
		\toprule[0.75pt]
		{\diagbox[width=15mm,trim=l]{Methods}{\#Outliers}}&& 0 & 4 & 8  & 12  & 16 & 20 \\
		\hline
		\rule{0pt}{9pt}&\hspace{-1.6mm}{GA~\cite{[1996-Gold]}}
		&0.31  &0.80   &1.21  &1.74  &2.29  &2.78 \\
		\rule{0pt}{9pt}&\hspace{-3mm}{RRWM~\cite{[2010-Cho-eccv]}}
		&0.04  &0.07   &0.12  &0.18  &0.24  &0.31\\
		\rule{0pt}{9pt}&\hspace{-2mm}{MPM~\cite{[2014-Cho-cvpr]}}
		&0.35  &0.61   &0.94  &1.40  &2.06  &3.05\\
		\rule{0pt}{9pt}&\hspace{-3mm}{FRGM~\cite{[2019-FRGM]}}
		&0.44  &0.61   &0.78  &0.96  &1.14  &1.36\\
		\rule{0pt}{9pt}&\hspace{-3mm}{BPFG~\cite{[2018-Wang]}}
		&1.07  &23.84  &37.79 &61.04 &83.41 &122.59 \\
		\rule{0pt}{9pt}&\hspace{-3mm}{FGMD~\cite{[2016-Zhou-pami]}}
		&0.68  &10.01  &12.67 &15.44 &19.47 &24.21 \\
		\hline
		\rule{0pt}{9pt}&\hspace{-0.1mm} {\bf ZAC}
		&0.18  &0.25   &0.32  &0.39  &0.47  &0.56\\
		\rule{0pt}{9pt}&{\hspace{-0.6mm} \bf ZACR}
		&0.53  &0.75   &0.89  &1.05  &1.20  &1.36\\
		\bottomrule[0.75pt]
	\end{tabular}
	\vspace{-2mm}
	\caption{Average running time (s) w.r.t. number of outliers.}
	\label{tab:pascal_time}
\end{table}

Next, we compared all the methods in terms of matching accuracy and time consumption. For overall comparisons, we set a series of numbers $k=\lfloor ratio\cdot\min\{m,n\}\rfloor$ ($ratio=0.3,0.35,...,1$ such that $k\geq 5$ since $m,n\in[15,75]$) in feasible fields $\hat{\mathcal{P}}_k$ for our method. And then, we also ran the compared methods with their soft-assignment matrix and evaluated their matching accuracy with the top $k$ matchings. Note that, since the methods FGMD~\cite{[2016-Zhou-pami]} and BPFG~\cite{[2018-Wang]} only obtain binary correspondences, we can only compute their matching accuracy with top $k=1\cdot\min\{m,n\}$ matchings.

Fig.~\ref{fig:pascal_fm_time} shows the highest average F-measure of all methods w.r.t the numbers of outliers. We can see that our methods ZAC and ZACR are more robust to outliers.  
Particularly, as shown in Fig.~\ref{fig:pascal_pr}, with a wide range of $ratio$, our methods achieve much higher precision, which means that the proposed outlier identification and removal approach can efficiently reduce incorrect or redundant matchings. Tab.~\ref{tab:pascal_time} reports the average time consumption, our methods take acceptable and intermediate time. Since the regularization term in Eq.~\eqref{eq:obj_reg} is more flexible than the equation constraint $\mathbf{1}^{\text{T}}\mathbf{P1}=k$, ZACR has a little higher accuracy than ZAC. However, as mentioned in Sec.~\ref{sec:optimiza}, since ZAC solves kLAP while ZACR uses LP solver, ZAC runs much faster than ZACR. Overall, ZAC achieves better trade-off between matching accuracy and time consumption than ZACR.

\begin{figure*}[!htb]
	\centering
	{\includegraphics[width=0.6\linewidth]{./image/pascal_legend_pr.pdf}}
	
	\vspace{-1mm}\subfigure[Pair 1-2]
	{\includegraphics[width=0.195\linewidth]{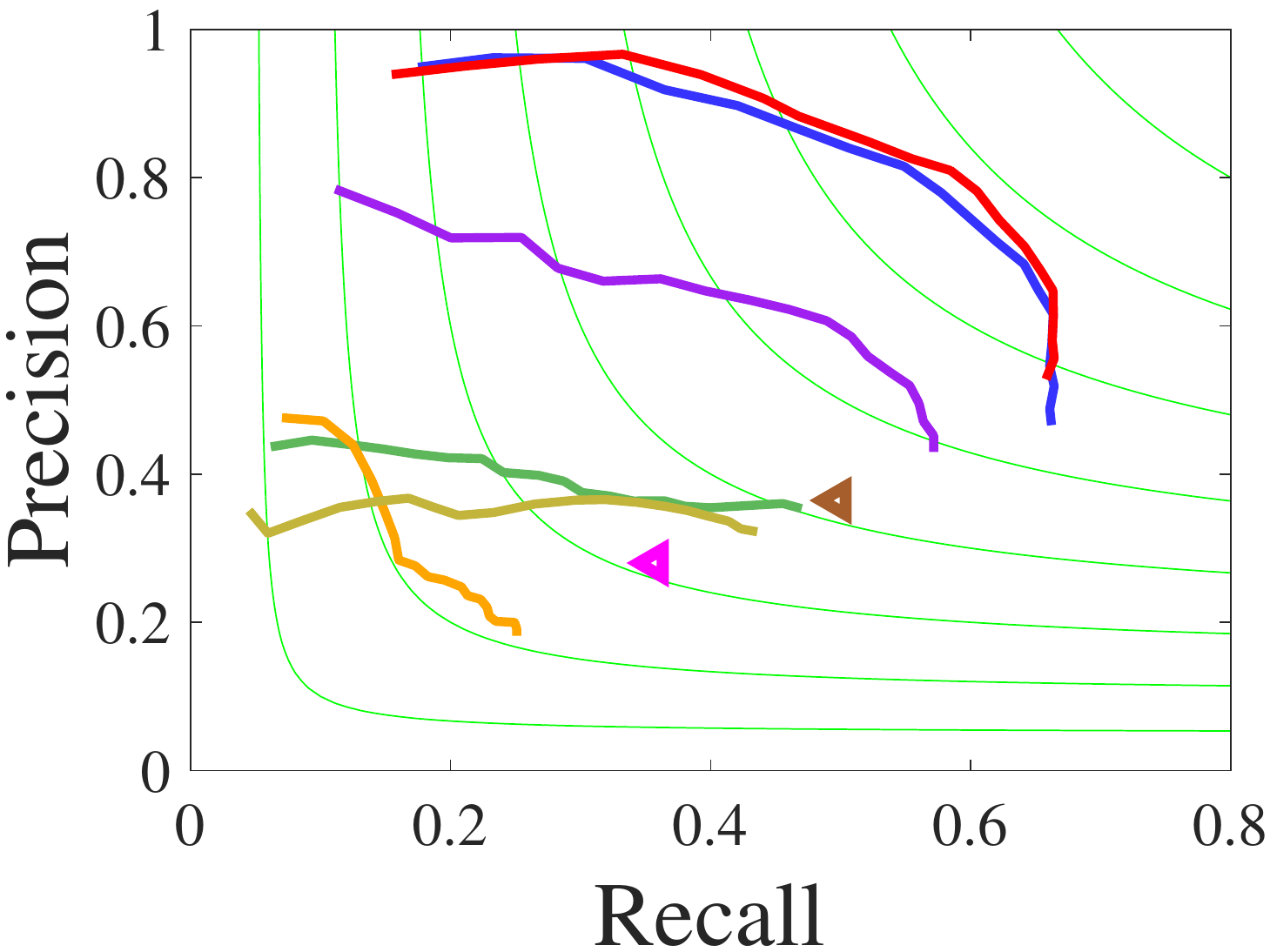}}
	\subfigure[Pair 1-3]
	{\includegraphics[width=0.195\linewidth]{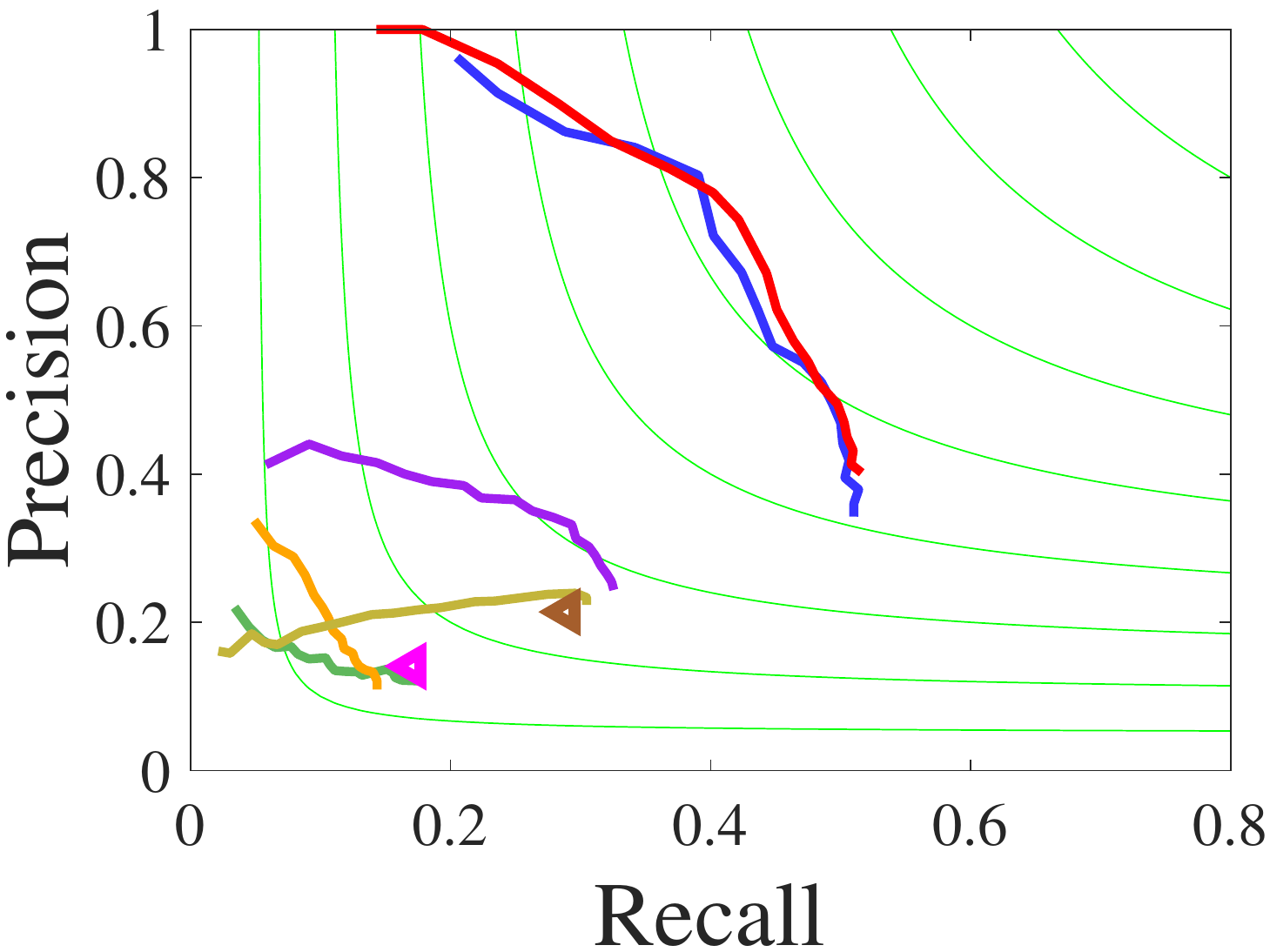}} 
	\subfigure[Pair 1-4]
	{\includegraphics[width=0.195\linewidth]{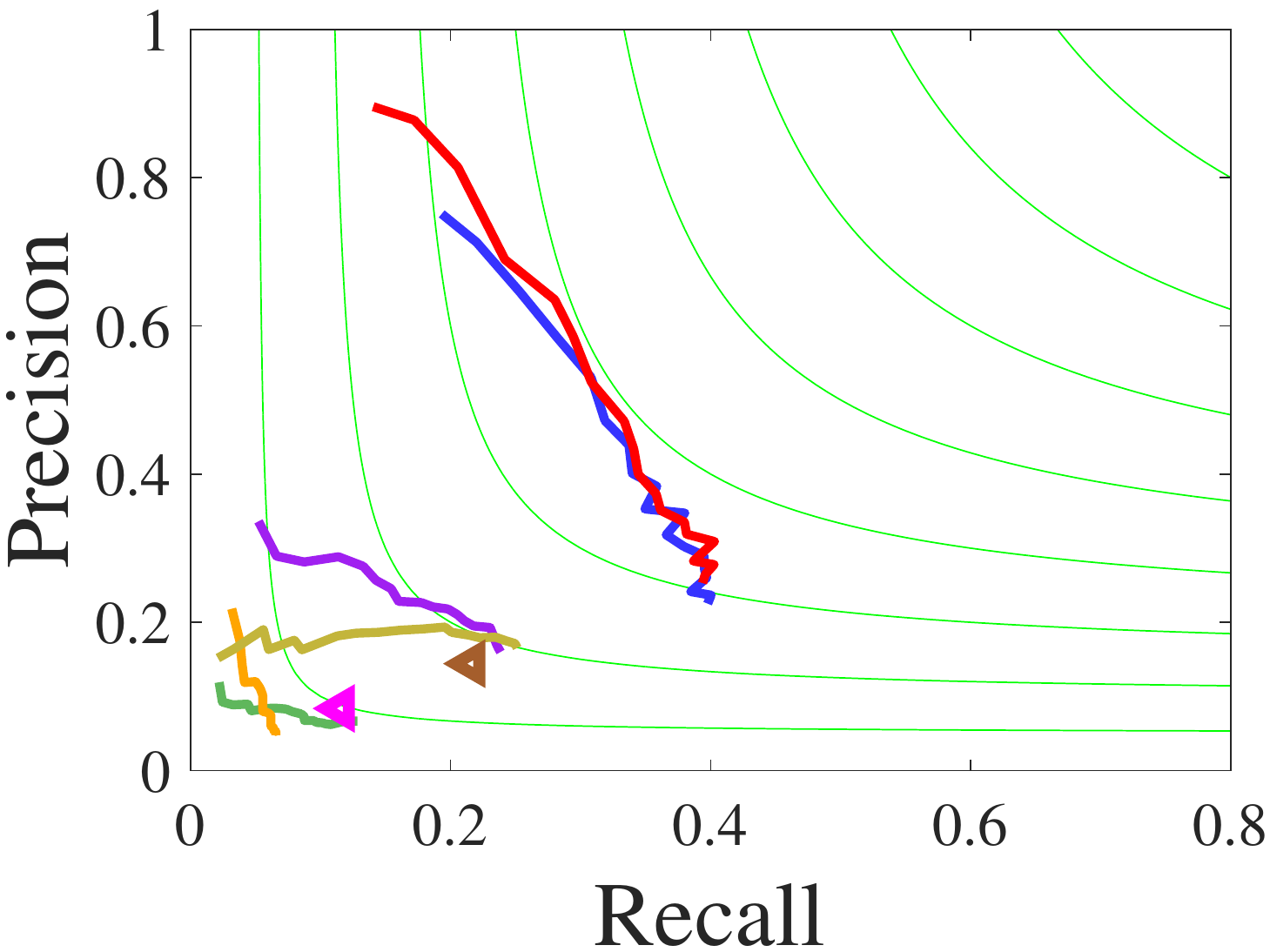}}
	\subfigure[Pair 1-5]
	{\includegraphics[width=0.195\linewidth]{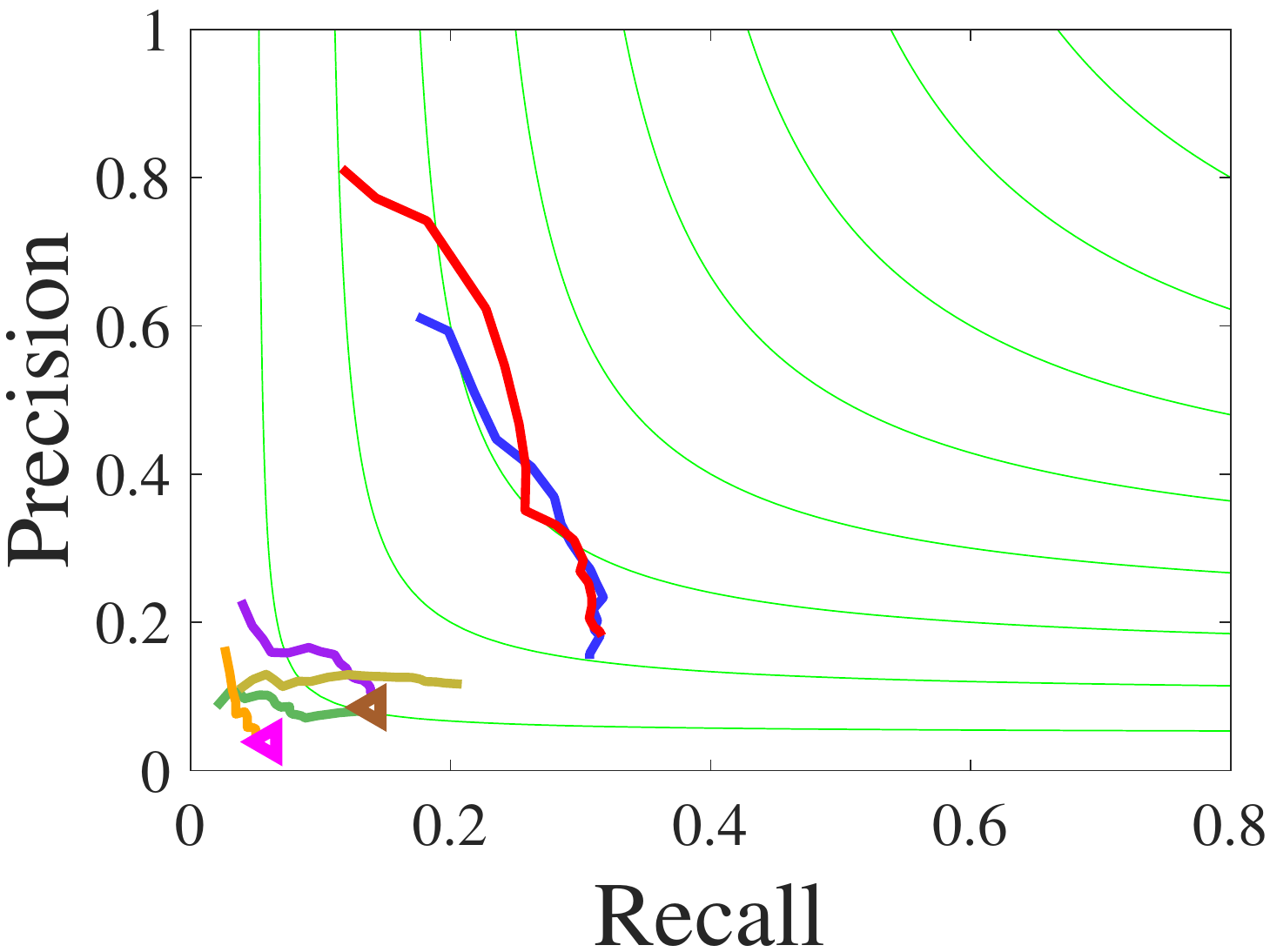}}
	\subfigure[Pair 1-6]
	{\includegraphics[width=0.195\linewidth]{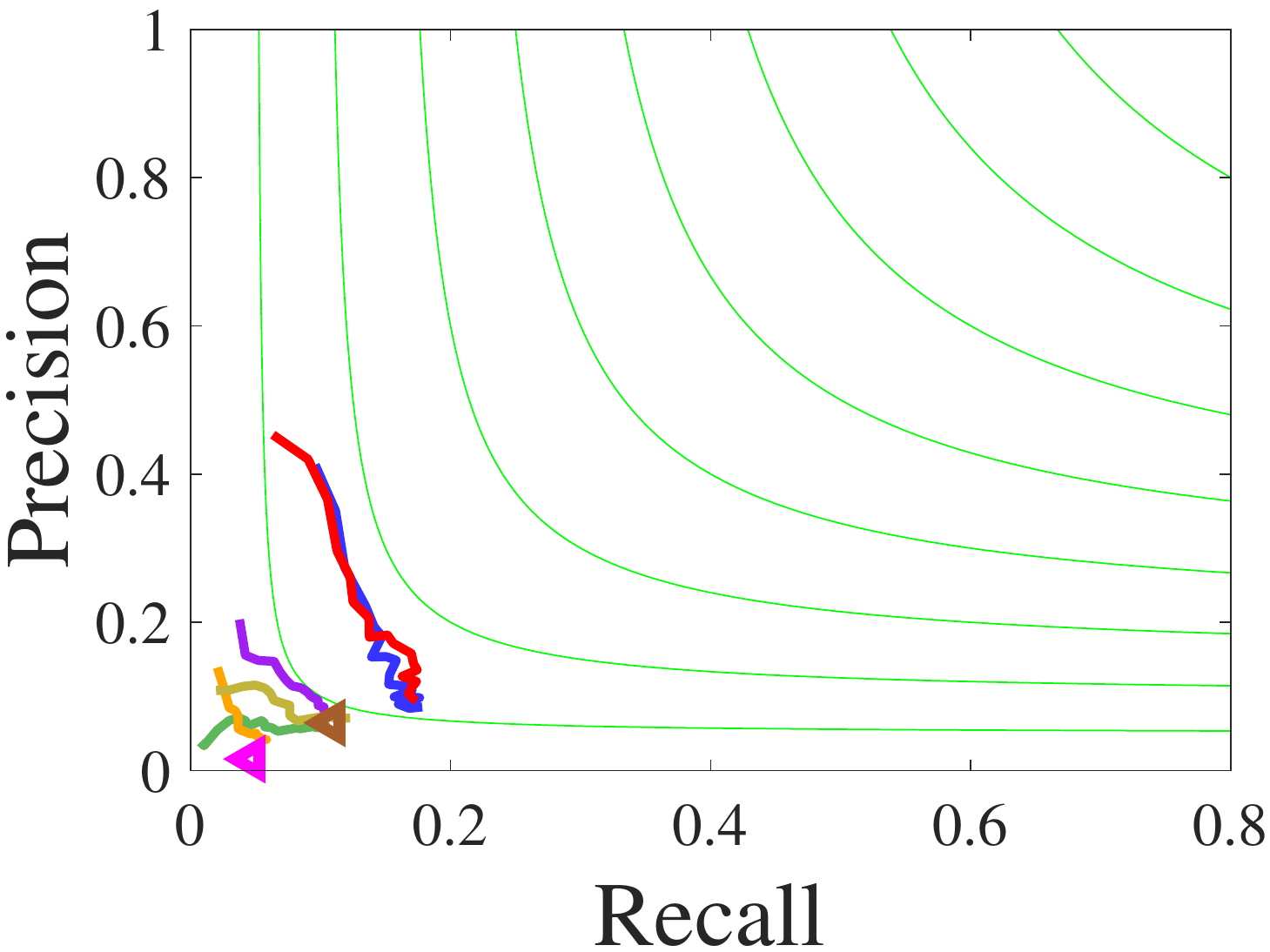}}
	\vspace{-4mm}
	\caption{Average recall (\%) and precision (\%) of all the methods with $ratio=0.1,0.15,...,1$. From pair 1-2 to 1-6, the graph pairs become more challenging for graph matching.}
	\label{fig:sift_pr}
\end{figure*}

\subsection{Results on VGG dataset} \label{sec:vgg}

As the example shown in Fig.~\ref{fig:fig1_com}, graph pairs in PASCAL dataset are generated with similar shapes. Thus, the experiments above evaluate the performance of all the methods in terms of shape consistency. Furthermore, we conducted experiments on more practical dataset to evaluate all the GM methods with more complicated graphs under varying geometric or physical factors. 

We adopted the widely used VGG dataset\footnote{{\scriptsize\url{http://www.robots.ox.ac.uk/~vgg/research/affine/}}} that consists of 8 groups of images (with sizes near 1000$\times$1000) and each group has 6 images with varying blurring, viewpoint, rotation, light, zoom and JPEG compression (see examples in supplementary material). For each group, there exist 5 affine matrices $\mathbf{H}_{1s}\in\mathbb{R}^{3\times3}$, $s=2,...,6$ that represent the ground-truth affine transformation from image 1 to images 2--6, respectively. We first formed graph pairs $\mathcal{G},\mathcal{G}'$ between image 1 and images 2--6 in each group. Then, we utilized feature detector SIFT~\cite{[2004-Lowe]} to generate nodes of graphs. Note that, since the compared methods FGMD, BPFG and MPM were highly time consuming with large-scale complete graphs, we adjusted the threshold of SIFT such that the numbers of output features were around 100 and neglected repeated features. We computed the settings as the same as in PASCAL dataset except that $\{\mathbf{v},\mathbf{v}_a'\}$ were SIFT features and set $\lambda_1=10,\lambda_2=1$.

\begin{figure}
	\centering
	\subfigure[Average recall and precision]
	{\includegraphics[width=0.47\linewidth]{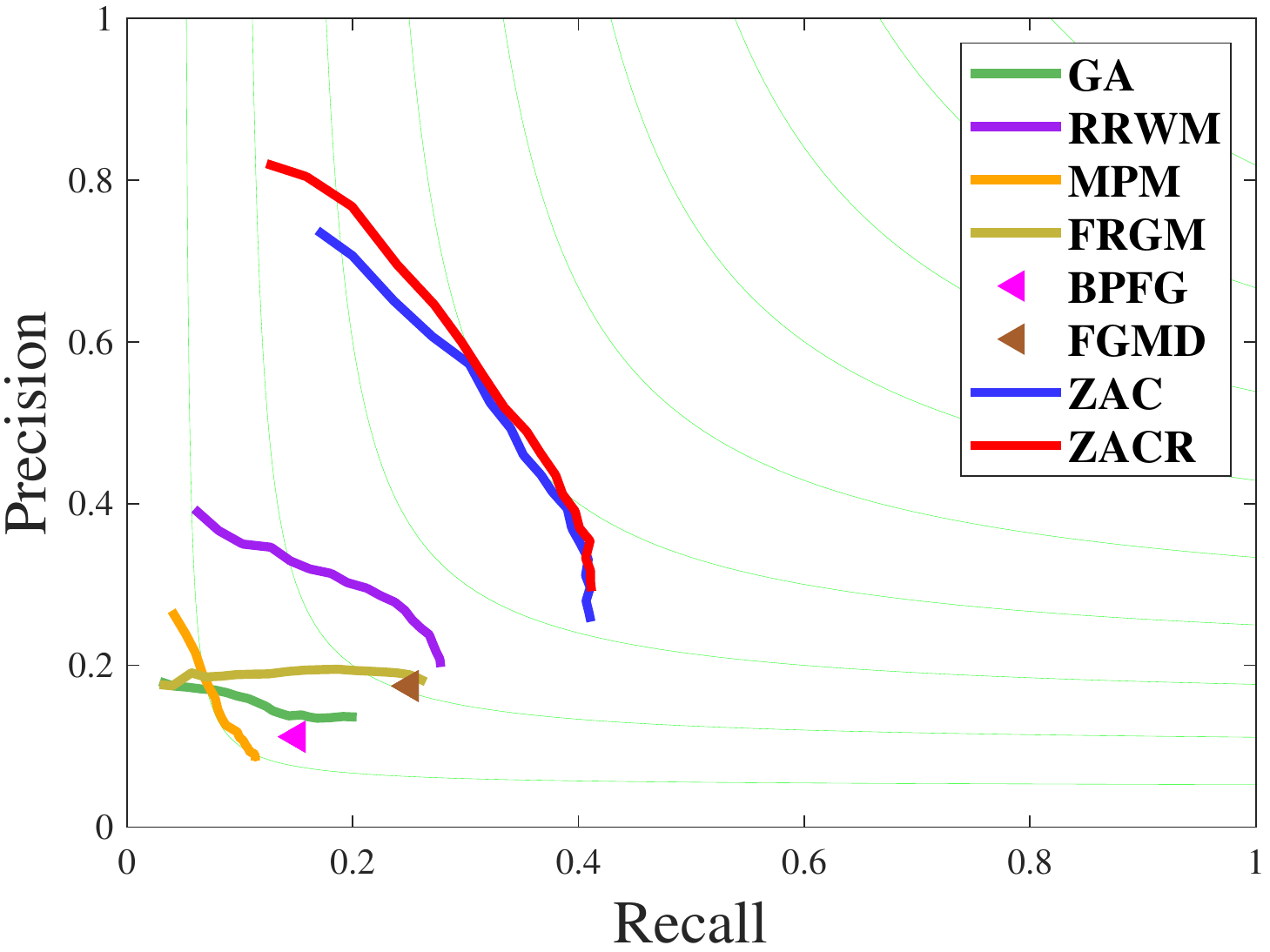}}
	\hspace{3mm}\subfigure[ Time consumption]
	{\includegraphics[width=0.47\linewidth]{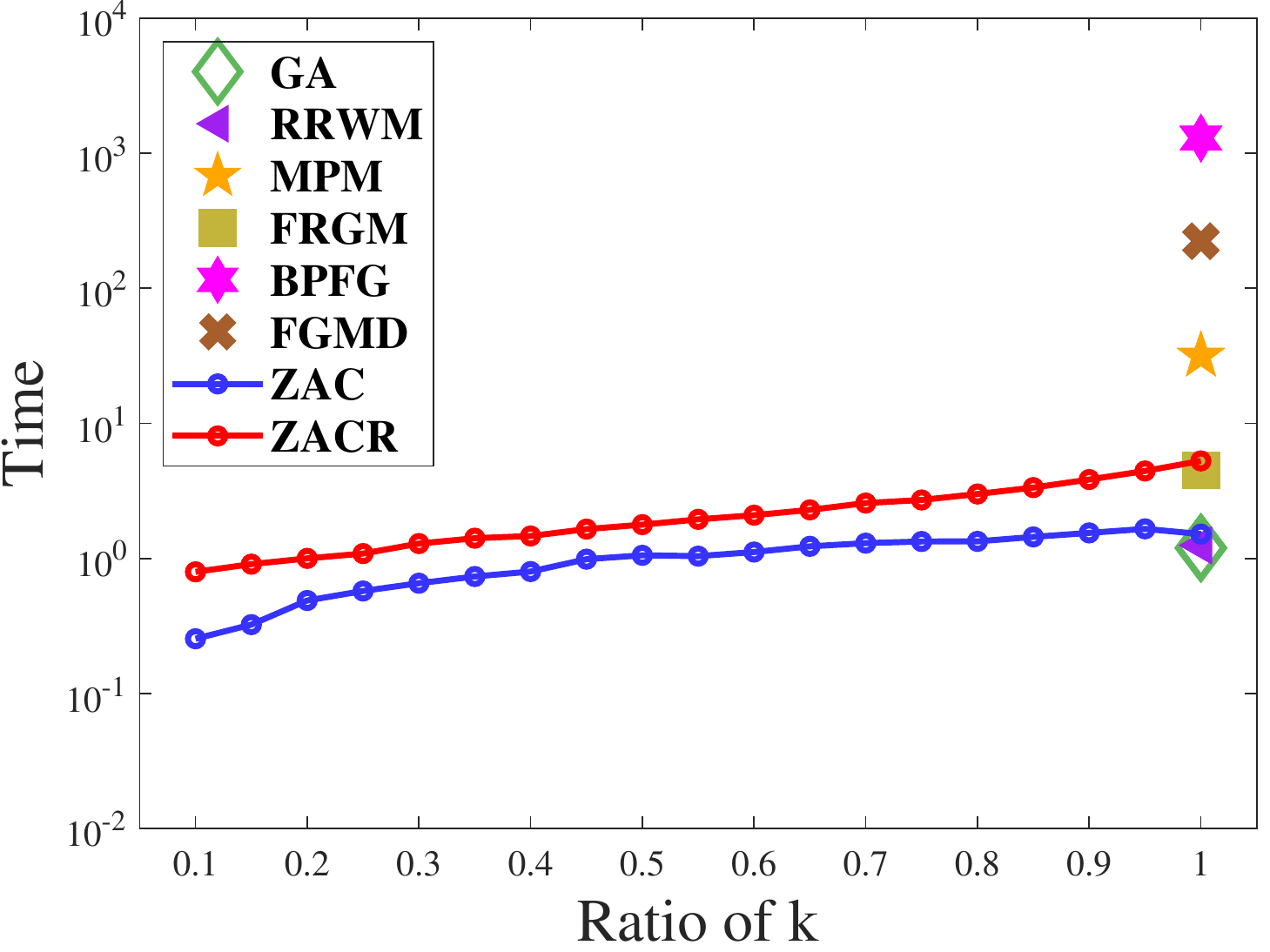}}
	\vspace{-3mm}
	\caption{ Average recall (\%), precision (\%) and time consumption (s) of all the methods on VGG dataset.}
	\label{fig:sift_pr_all}
\end{figure}

An output matching result was evaluated as follows:  for each node $V_i\in\mathcal{G}$ matched with $V'_{a_i}\in\mathcal{G}'$, we calculated its correct correspondence $V'_{\sigma_i}=\mathbf{H}_{1s}V_i$ using the ground-truth affine matrix $\mathbf{H}_{1s}$. Then, if the distance $||V'_{a_i}-V'_{\sigma_i}||$ was less than 10 pixels, the matching between $V_i$ and $V'_{a_i}$ was accepted as a correct matching. We set $k$ with varying $ratio=0.1,0.15,...,1$ for the evaluation of recall and precision. Moreover, we also evaluated our time consumption w.r.t the varying $ratio$, since the number of nodes in refined graphs obtained by our outlier identification and removal approach will be influenced by $ratio$. Note that, the time consumption of the other methods will not be affected by $ratio$ since they match all the nodes in graphs.

As shown in Fig.~\ref{fig:sift_pr}, under the varying geometric or physical conditions, our methods ZAC and ZACR can achieve much higher recall and precision. Fig.~\ref{fig:sift_pr_all} (a) shows the overall average matching accuracy and time consumption w.r.t. varying $ratio$, our method ZAC and ZACR have much better matching accuracy within much less time consumption, even though on complicated graphs with numerous outliers and varying geometric or physical factors in practice.

\subsection{Deformable graph matching}

Deformable graph matching (DGM)~\cite{[2006-Caetano],[2006-Zheng],[2016-Zhou-pami],[2019-FRGM]} is an important subproblem of GM, which focuses on incorporating rigid or non-rigid deformations between graphs. The main idea is to estimate both the correspondence $\mathbf{P}$ and deformation parameters $\tau$ by minimizing the sum of residuals
\begin{equation}\label{eq:dgm}
\min_{\mathbf{P},\tau} J(\mathbf{P},\tau)=\sum_{i,a}\mathbf{P}_{ia}||V'_a-\tau(V_i)||^2 + \lambda_r\Upsilon(\tau),
\end{equation}where ${V}_i,{V}'_a\in\mathbb{R}^d$ are the nodes in $\mathcal{V},\mathcal{V}'$, and $\Upsilon(\cdot)$ is a regularization term. Generally, the rigid or non-rigid deformation is parameterized as $\tau(\mathcal{V})=s\mathcal{V}\mathbf{R}+t$ or $\tau(\mathcal{V})=\mathcal{V}+\mathbf{W}\mathbf{G}$. See~\cite{[2016-Zhou-pami],[2019-FRGM]} for more comprehensive reviews. 

\begin{figure}[!htb]
	\centering
	\subfigure[Rigid deformation]
	{\includegraphics[width=0.46\linewidth]{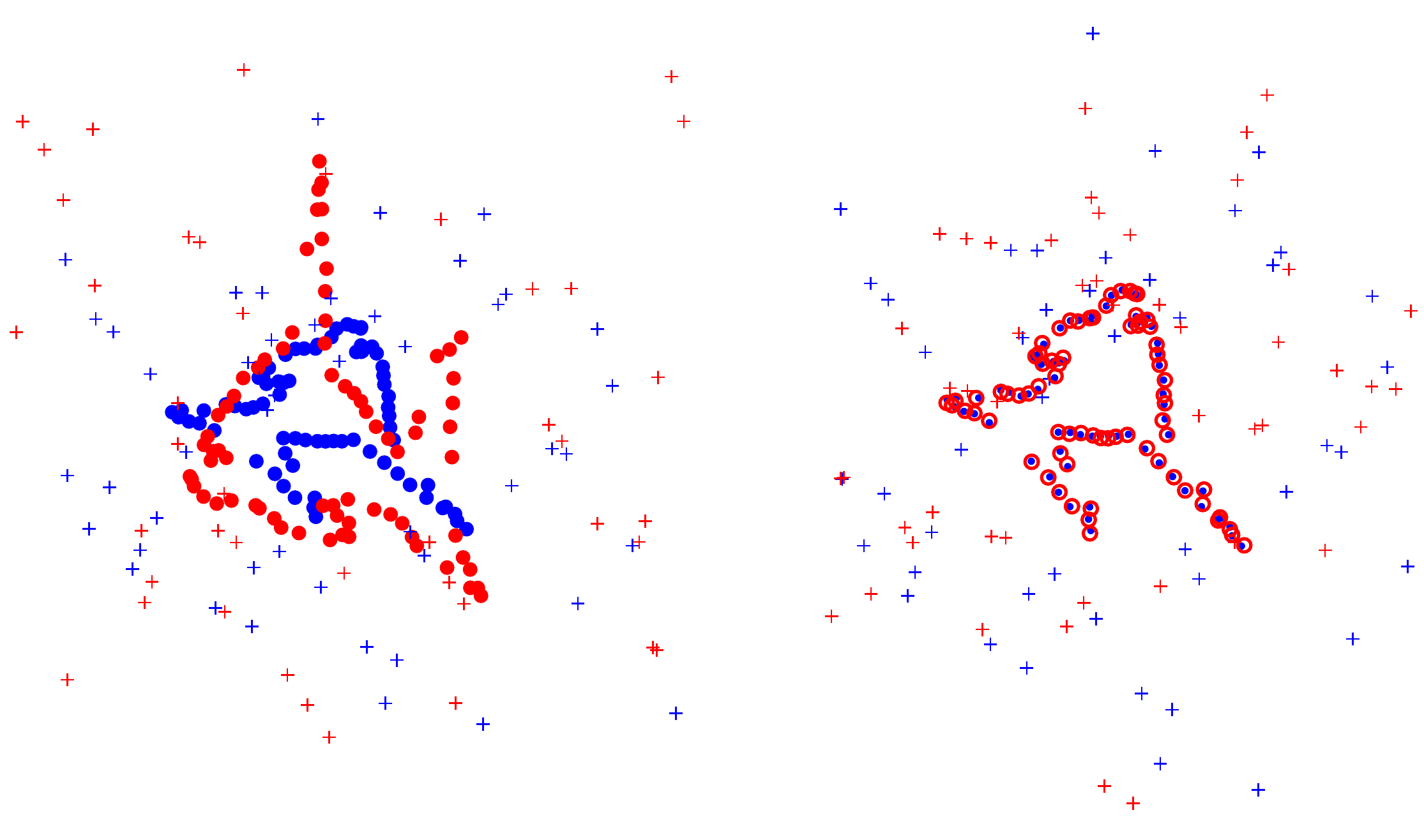}} 
	\hspace{5mm}\subfigure[Non-rigid deformation]
	{\includegraphics[width=0.46\linewidth]{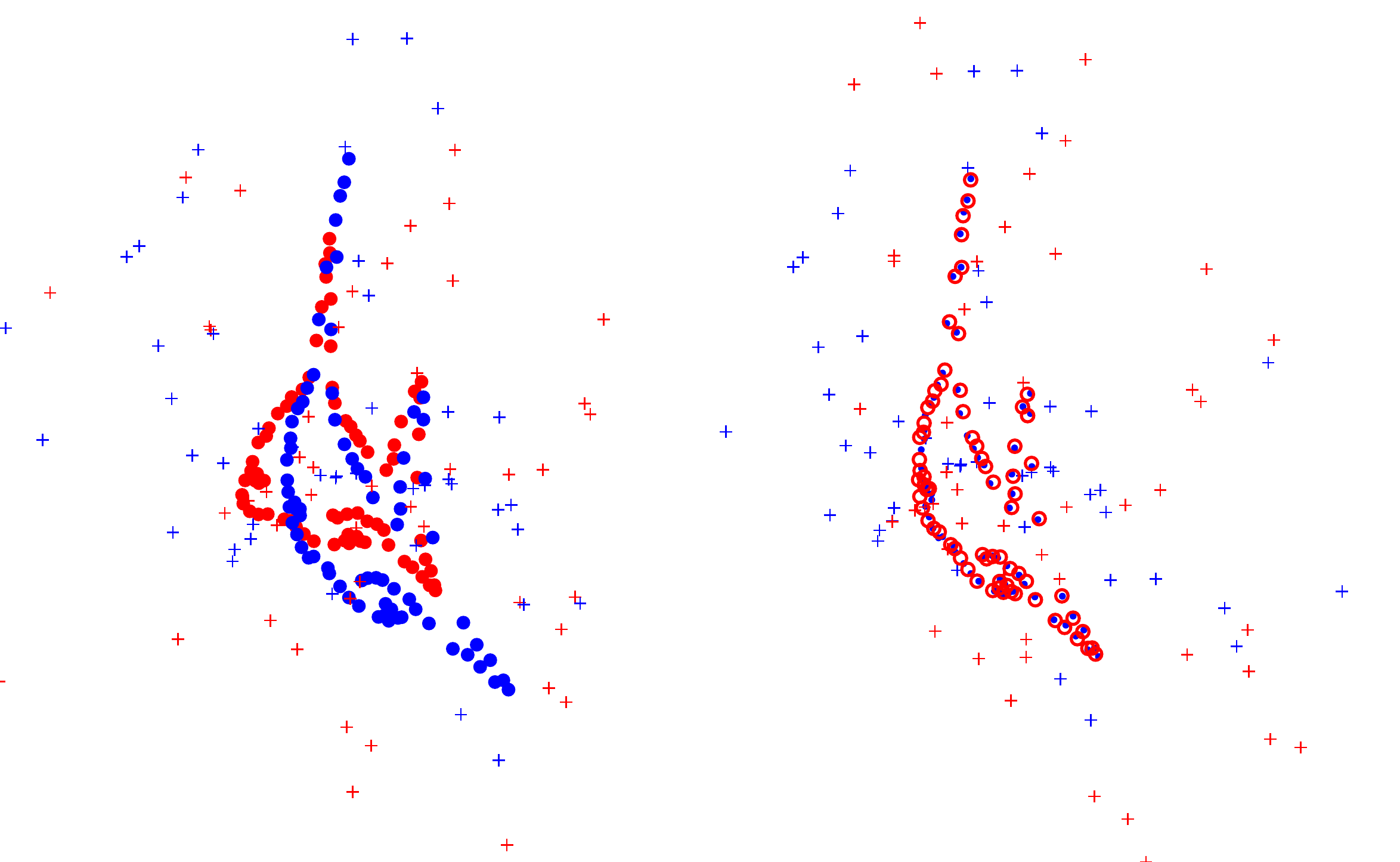}} 
	\vspace{-4mm}
	\caption{Examples of deformable graph matching results of our method ZAC on the graphs under geometric deformations, to which the noises and outliers are also added.}
	\label{fig:point_reg_demo}
	\vspace{-2mm}
\end{figure}
\begin{figure}
	\centering
	{\includegraphics[width=0.8\linewidth]{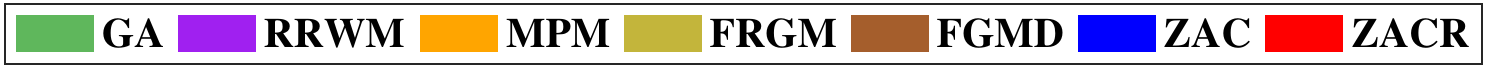}}	
	{\includegraphics[width=0.99\linewidth]{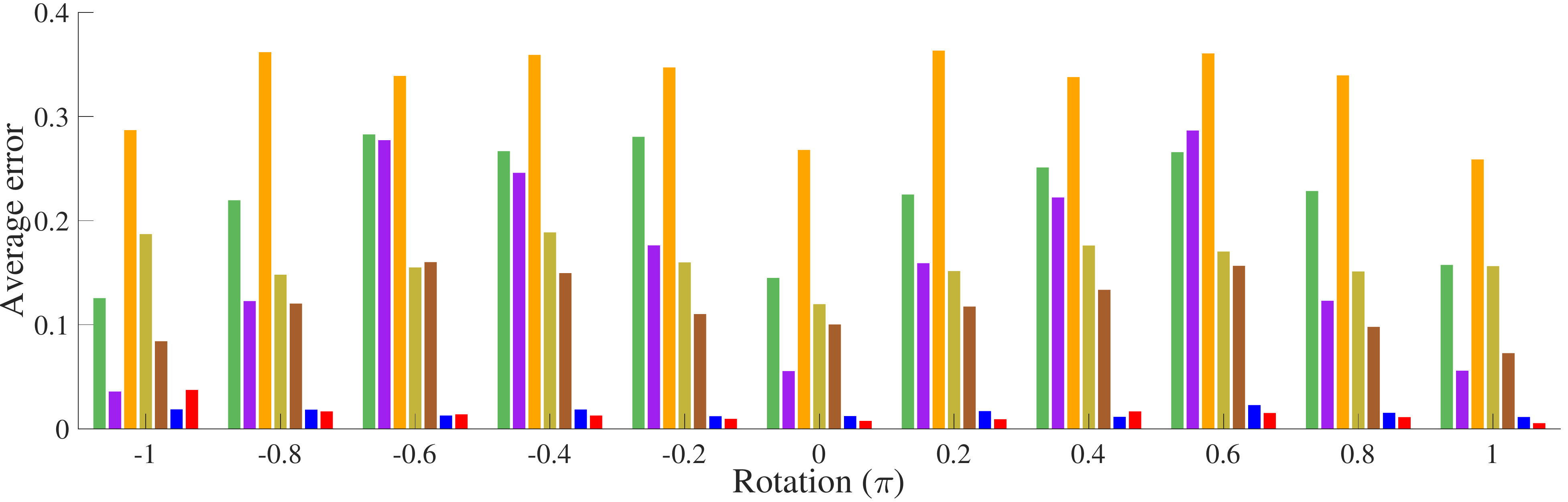}} 
	\vspace{-3mm}
	\caption{Average errors w.r.t. varying rotations.}
	\label{fig:point_reg1}
\end{figure}
\begin{figure}
	\centering
	{\includegraphics[width=0.8\linewidth]{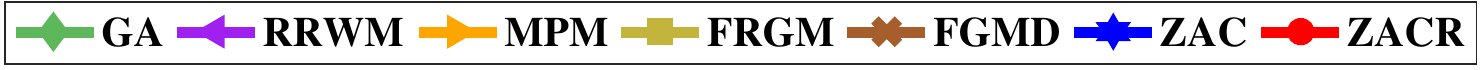}}	
	{\includegraphics[width=0.48\linewidth]{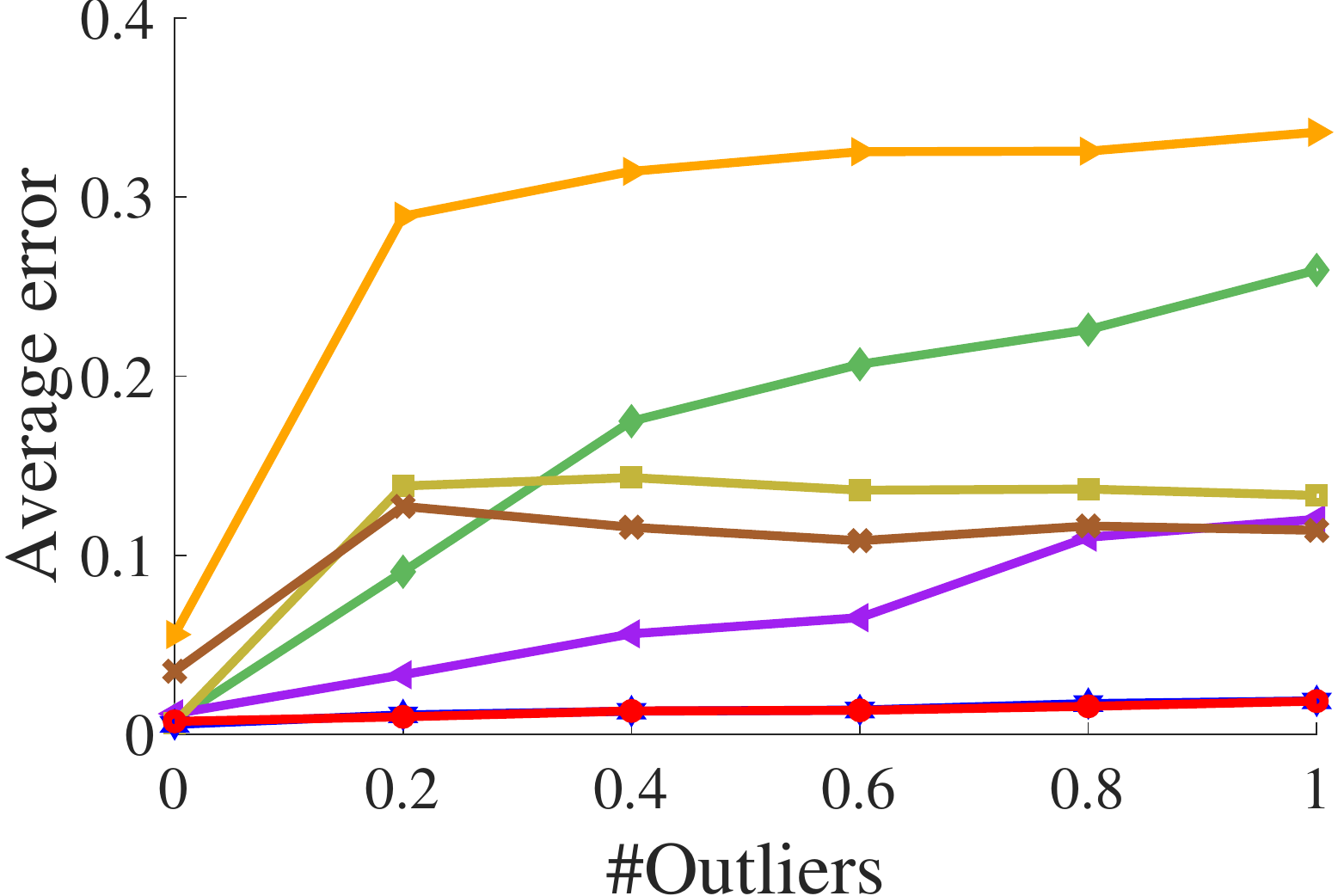}}
	\hspace{2mm}{\includegraphics[width=0.48\linewidth]{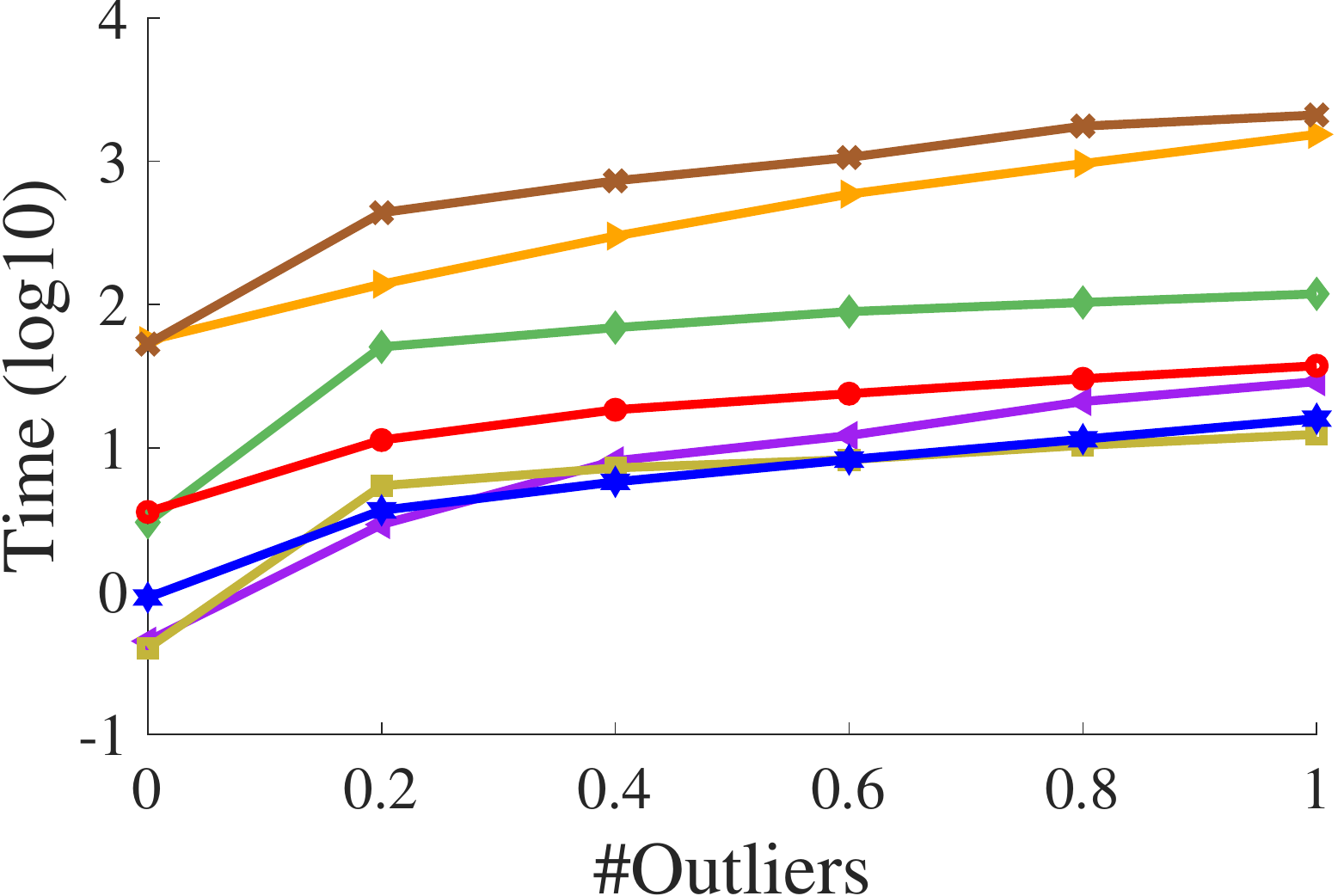}}
	\vspace{-5mm}
	\caption{Average errors and time consumptions (s) w.r.t. varying outliers.}
	\label{fig:point_reg2}
\end{figure}


Finding correct correspondence $\mathbf{P}$ plays the central role for solving Eq.~\eqref{eq:dgm}. Once $\mathbf{P}$ is well-estimated, the geometric parameter $\tau$ can be solved with closed form~\cite{[2016-Zhou-pami],[2019-FRGM]}. In this section, we applied all the GM methods with $k=0.5\times\min\{m,n\}$ to find $\mathbf{P}$ and then computed $\tau$. We iteratively executed this procedure till it converges. All the GM methods applied the same settings used in Sec.~\ref{sec:vgg} except that the node attributes $\mathbf{v}_i,\mathbf{v}'_a$ were shape context here.

We adopted the widely used 2D shape templates in~\cite{[2010-Myronenko],[2016-Ma],[2016-Zhou-pami],[2019-FRGM]} for evaluation and comparison. We uniformly sampled $50\%$ points of the shape template as inliers of $\mathcal{G}$ and $\mathcal{G}'$. And slight noises with uniform distribution $U(0,0.01)$ were also added to $\mathcal{G}$ and $\mathcal{G}'$. 

In this section, we conducted two series of experiments on graphs with varying deformations and outliers. First, we measured the robustness of each algorithm to rotations in rigid deformation. We rotated $\mathcal{G}'$ with varying degrees in $[-\pi,\pi]$ and then randomly added 10--50 outliers with Gaussian distribution $N(0,0.5)$ to both $\mathcal{G}$ and $\mathcal{G}'$. Second, we evaluated the robustness to outliers. For the rigid deformation, we randomly rotated $\mathcal{G}'$ with degrees in $[-0.1\pi,0.1\pi]$. For the non-rigid deformation, we deformed $\mathcal{G}'$ following the settings in~\cite{[2016-Zhou-pami],[2019-FRGM]} by weight matrices $\mathbf{W}$ with Gaussian distribution $N(0,0.5)$. And then, we incrementally added $\{0\%,20\%,...,100\%\}\cdot\#inliers$ numbers of outliers with Gaussian distribution $N(0,0.5)$ to $\mathcal{G}$ and $\mathcal{G}'$. See examples in Fig.~\ref{fig:point_reg_demo}. For all GM methods, we adopted the rotation-invariant shape context advised by~\cite{[2016-Ma],[2019-FRGM]} for rigid deformations. For evaluation, we computed the average error between the transformed inliers $\{\tau(V_i)\}$ and their ground-truth matching point $\{V'_{\delta_i}\}$, {\ie}, $\frac{1}{\#inliers}\sum_i||\tau(V_i)-V'_{\delta_i}||$.

Fig.~\ref{fig:point_reg1} shows the average errors of all the methods w.r.t. varying rotations. We can see that our methods ZAC and ZACR can nearly perfectly recover and match all the  graphs across all the rotations. Fig.~\ref{fig:point_reg2} reports the average errors and time consumptions on graphs with rigid and nonrigid deformations. Our methods have much less average error than all the other methods and runs faster than most of them. Note that, we did not compare BPFG due to its extremely unacceptable time consumption (more than 5 hours to match only one pair of graphs) in this experiment.

For more comprehensive evaluation, we also compared with two efficient point registration algorithms GLS~\cite{[2016-Ma]} and CPD~\cite{[2010-Myronenko]}, which can address the deformable graph matching problem from the perspective of point registration. The comparison results are shown in Tab.~\ref{tab:point_gls_cpd}, our methods ZAC and ZACR achieve comparable results with GLS and CPD for graphs without outliers, and have less average errors for complicated graphs with numerous outliers.

\begin{table}[!htb]
	\centering
	\scriptsize
	\begin{tabular}{m{0pt}p{25pt}|p{14pt}p{14pt}p{14pt}p{14pt}p{14pt}p{16pt}|p{5pt}|}
		\toprule[0.75pt]
		{\diagbox[width=15mm,trim=l]{Methods}{\#Outliers}}&& 0\% & 20\% & 40\%  & 60\%  & 80\% & 100\%& \\
		\hline
		\rule{0pt}{9pt}&\hspace{-2mm}{GLS~\cite{[2016-Ma]}}
		&\bf0.005  &0.006   &0.011   &0.021    &0.030  &0.044 &\\
		\rule{0pt}{9pt}&\hspace{-2mm}{CPD~\cite{[2010-Myronenko]}}
		&\bf0.005  &0.052   &0.085   &0.112   &0.124   &0.135 &\multirow{3}{2pt}{\rotatebox{-90}{\hspace{-3mm} Rigid}}\\
		\cline{1-8}
		\rule{0pt}{9pt}&{\hspace{-0mm} ZAC}
		&\bf0.005  &\bf0.005&\bf0.005&\bf0.005&\bf0.006&\bf0.006&\\
		\rule{0pt}{9pt}&{\hspace{-0.5mm} ZACR}
		&0.009    &\bf0.005 &0.006   &\bf0.005&\bf0.006&\bf0.006&\\
		\hline
		\hline	
		\rule{0pt}{9pt}&\hspace{-2mm}{GLS~\cite{[2016-Ma]}}
		&0.010     &{\bf0.014}&0.043    &0.050     &0.064     &0.092 
		&{\multirow{3}{2pt}{\rotatebox{-90}{\hspace{0.0mm} Non-rigid}}}\\
		\rule{0pt}{9pt}&\hspace{-2mm}{CPD~\cite{[2010-Myronenko]}}
		&\bf0.006  &0.015    &0.042     &0.065     &0.069     &0.083 &\\
		\cline{1-8}
		\rule{0pt}{9pt}&\hspace{-0mm} {ZAC}
		&{\bf0.006}&0.016    &{\bf0.021}&{0.022}&{0.028}&{\bf0.032}&\\
		\rule{0pt}{9pt}&{\hspace{-0.5mm} ZACR}
		&\bf0.006  &\bf0.014 &\bf0.021  &\bf0.021  &\bf0.025  &\bf0.032&\\
		\bottomrule[0.75pt]
	\end{tabular}
	\vspace{-2mm}
	\caption{Average errors w.r.t. varying outliers on graphs with rigid and non-rigid deformations.}
	\label{tab:point_gls_cpd}
\end{table}

\section{Conclusion}
This paper presents the zero-assignment constraint to address the problem of graph matching in the presence of outliers. Beyond the empirical criterion, we propose both theoretical foundations and quantitative analyses for this problem, on which bases we are inspired to construct reasonable objective function and find out the sufficient condition for its rationality. Moreover, we propose an efficient algorithm consisting of fast optimization and outlier identification, which ensures us to handle complicated graphs with numerous cluttered outliers in practice and achieve state-of-the-art performance in terms of accuracy and efficiency. In future work, we will go further to consolidate the theoretical foundation of graph matching problem with outliers by extending the zero-assignment constraint to the famous QAP formulations of graph matching in Eq.~\eqref{eq:gm=lawler} and Eq.~\eqref{eq:gmKoom22}.

\section*{Acknowledgement}
This work was supported by the National Natural Science Foundation of China under Grant 61771350 and Grant 61922065.


{\bibliographystyle{ieee_fullname}
\bibliography{egbib_bpgm}
}

\newpage
\mbox{}
\newpage


\section*{Supplementary material}

# Supplementary Material for
# Zero-Assignment Constraint for Graph Matching with Outliers

*In the following supplementary document, we present both the detailed proof for proposition 4 in our manuscript and more comparison results. To make it more intuitive and easy to understand, we will present the proof accompanied with an example shown in Fig. 1.*

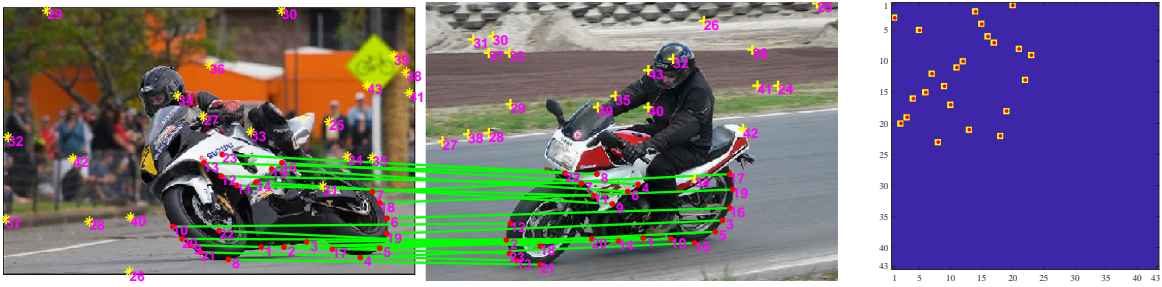

Figure 1: Left: the visualization of the ideal matching $\mathbf{P}^*$ between two graphs. Right: the ideal ground-truth correspondence matrix $\mathbf{P}^*$ (red markers) and a specific correspondence matrix $\mathbf{P} = \mathbf{P}^*$ (yellow markers).

## 1. Proof for proposition 4

The two graphs $\mathcal{G}$ and $\mathcal{G}'$ in Fig. 1 consist of 23 inliers (red dots) and 20 outliers (yellow signs), respectively. Without ambiguity, we can rerange the inliers as the first $k = 23$ nodes in the two graphs to make the visualization more clear. Therefore, the index sets $\mathscr{A}_I, \mathscr{A}_O, \mathscr{B}_I, \mathscr{B}_O$ defined in our manuscript can be specifically written as

$$\mathscr{A}_I = \{1, 2, ..., 23\}, \quad \mathscr{A}_O = \{24, 25, ..., 43\}, \quad \mathscr{B}_I = \{1, 2, ..., 23\}, \quad \mathscr{B}_O = \{24, 25, ..., 43\}. \tag{1}$$

The numbers shown in pink are the indices of nodes in the two graphs. As we can see, the ideal ground-truth correspondence can be expressed as

$$\begin{array}{cccccccccccccccccccccccc}
1 & 2 & 3 & 4 & 5 & 6 & 7 & 8 & 9 & 10 & 11 & 12 & 13 & 14 & 15 & 16 & 17 & 18 & 19 & 20 & 21 & 22 & 23 \\
\downarrow & \downarrow & \downarrow & \downarrow & \downarrow & \downarrow & \downarrow & \downarrow & \downarrow & \downarrow & \downarrow & \downarrow & \downarrow & \downarrow & \downarrow & \downarrow & \downarrow & \downarrow & \downarrow & \downarrow & \downarrow & \downarrow & \downarrow \\
20 & 14 & 1 & 15 & 5 & 16 & 17 & 21 & 23 & 12 & 11 & 7 & 22 & 9 & 6 & 4 & 10 & 19 & 3 & 2 & 13 & 18 & 8
\end{array}.$$

To make the proof more clear, we first define some basic notations as follows. Given arbitrary partial permutation $\tau : \mathscr{A} \to \mathscr{B}$ and its compatible partial permutation matrix (*i.e.*, correspondence matrix) $\mathbf{P} \in \mathcal{P}_k$, we define that,

$$\mathcal{A}_I^{\mathbf{P},0} \triangleq \{i \in \mathscr{A}_I; \tau(i) = \varnothing\}, \quad \mathcal{A}_I^{\mathbf{P},1} \triangleq \{i \in \mathscr{A}_I; \tau(i) \neq \tau^*(i)\}, \quad \mathcal{A}_I^{\mathbf{P},*} \triangleq \{i \in \mathscr{A}_I; \tau(i) = \tau^*(i)\},$$
$$\mathcal{A}_O^{\mathbf{P},0} \triangleq \{i \in \mathscr{A}_O; \tau(i) = \varnothing\}, \quad \mathcal{A}_O^{\mathbf{P},1} \triangleq \{i \in \mathscr{A}_O; \tau(i) \neq \varnothing\}.$$

These index sets can be equally defined as

$$\mathcal{A}_I^{\mathbf{P},0} \triangleq \{i \in \mathscr{A}_I; \mathbf{P}_{i,:} \equiv \mathbf{0}\}, \quad \mathcal{A}_I^{\mathbf{P},1} \triangleq \{i \in \mathscr{A}_I; \exists a \neq \tau^*(i), \mathbf{P}_{ia} = 1\}, \quad \mathcal{A}_I^{\mathbf{P},*} \triangleq \{i \in \mathscr{A}_I; \exists a = \tau^*(i), \mathbf{P}_{ia} = 1\},$$
$$\mathcal{A}_O^{\mathbf{P},0} \triangleq \{i \in \mathscr{A}_O; \mathbf{P}_{i,:} \equiv \mathbf{0}\}, \quad \mathcal{A}_O^{\mathbf{P},1} \triangleq \{i \in \mathscr{A}_O; \exists a \in \mathscr{B}, \mathbf{P}_{ia} = 1\}.$$

1

We can directly obtain some basic properties of the index sets above.

**Proposition 1.**

- *Complementary*

$$\mathcal{A}_I^{\mathbf{P},0} \cup \mathcal{A}_I^{\mathbf{P},1} \cup \mathcal{A}_I^{\mathbf{P},*} = \mathcal{A}_I, \quad \mathcal{A}_O^{\mathbf{P},0} \cup \mathcal{A}_O^{\mathbf{P},1} = \mathcal{A}_O. \tag{2}$$

- *Disjoint*

$$\mathcal{A}_I^{\mathbf{P},0} \cap \mathcal{A}_I^{\mathbf{P},1} = \mathcal{A}_I^{\mathbf{P},0} \cap \mathcal{A}_I^{\mathbf{P},*} = \mathcal{A}_I^{\mathbf{P},1} \cap \mathcal{A}_I^{\mathbf{P},*} = \varnothing, \quad \mathcal{A}_O^{\mathbf{P},0} \cap \mathcal{A}_O^{\mathbf{P},1} = \varnothing. \tag{3}$$

- *For the ideal ground-truth correspondence* $\mathbf{P}^*$

$$\mathcal{A}_I^{\mathbf{P}^*,0} = \mathcal{A}_I^{\mathbf{P}^*,1} = \varnothing, \quad \mathcal{A}_I^{\mathbf{P},*} \subseteq \mathcal{A}_I^{\mathbf{P}^*,*} = \mathcal{A}_I, \quad \mathcal{A}_O^{\mathbf{P},0} \subseteq \mathcal{A}_O^{\mathbf{P}^*,0} = \mathcal{A}_O, \quad \mathcal{A}_O^{\mathbf{P}^*,1} = \varnothing. \tag{4}$$

Intuitively, the definitions of sets $\mathcal{A}_I^{\mathbf{P},1}, \mathcal{A}_I^{\mathbf{P},0}, \mathcal{A}_O^{\mathbf{P},1}$ can be expressed as the disturbances to the ground-truth correspondence $\mathbf{P}^*$ caused by the given correspondence $\mathbf{P} \in \mathcal{P}_k$. As shown in Fig. 2, there are 4 typical kinds of disturbances.

- $\mathcal{A}_I^{\mathbf{P},1}$: disturbances resulting in incorrect matchings of inliers in $\mathscr{A}$. There are 2 kinds of disturbances: (I) incorrect matchings between inliers in $\mathscr{A}_I$ and inliers in $\mathscr{B}_I$, (II) incorrect matchings between inliers in $\mathscr{A}_I$ and outliers in $\mathscr{B}_O$.
- $\mathcal{A}_I^{\mathbf{P},0}, \mathcal{A}_O^{\mathbf{P},1}$: disturbances resulting in redundant matchings of outliers in $\mathscr{A}$. There are also 2 kinds of disturbances: (III) redundant matchings between outliers in $\mathscr{A}_O$ and inliers in $\mathscr{B}_I$, (IV) redundant matchings between outliers in $\mathscr{A}_O$ and outliers in $\mathscr{B}_O$. Note that, we have $|\mathcal{A}_I^{\mathbf{P},0}| = |\mathcal{A}_O^{\mathbf{P},1}|$: Eq. (2), (3) mean that $|\mathcal{A}_I^{\mathbf{P},0}| + |\mathcal{A}_I^{\mathbf{P},*}| + |\mathcal{A}_O^{\mathbf{P},1}| = k$; due to the constraint $\mathbf{1}^T \mathbf{P} \mathbf{1} = k$, it has to hold $|\mathcal{A}_I^{\mathbf{P},1}| + |\mathcal{A}_I^{\mathbf{P},*}| + |\mathcal{A}_O^{\mathbf{P},1}| = k$. Therefore, it derives $|\mathcal{A}_I^{\mathbf{P},0}| = |\mathcal{A}_O^{\mathbf{P},1}|$.

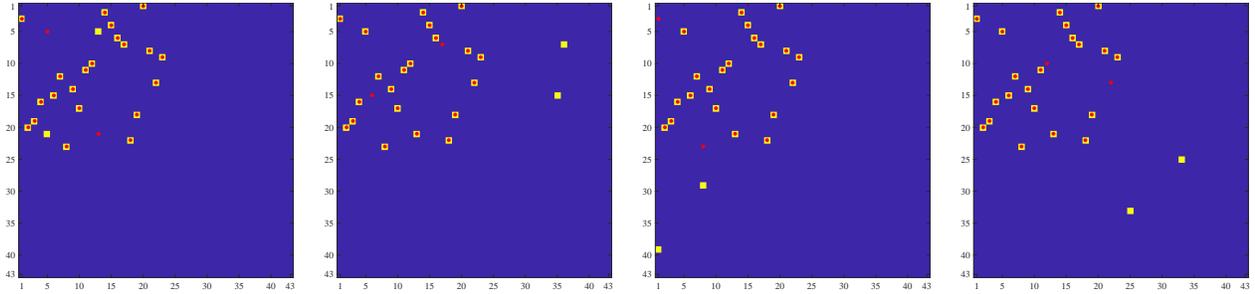

(a) Disturbance I     (b) Disturbance II     (c) Disturbance III     (d) Disturbance IV

Figure 2: The 4 typical kinds of disturbances cased by giving arbitrary correspondence matrices $\mathbf{P} \in \mathcal{P}_k$ (yellow signs). The red dots in each subfigure are the ideal ground-truth correspondence matrix $\mathbf{P}^*$.

As mentioned in Sec.3.4 in our manuscript, we first prove that $\forall \mathbf{P} \in \mathcal{P}_k, F_u(\mathbf{P}) \geq F_u(\mathbf{P}^*)$.

*Proof.*

$$F_u(\mathbf{P}) = \sum_{ia} \mathbf{D}_{ia} \mathbf{P}_{ia} = \sum_{\substack{i \in \mathcal{A}_I^{\mathbf{P},0} \\ a \in \mathscr{B}}} \mathbf{D}_{ia} \mathbf{P}_{ia} + \sum_{\substack{i \in \mathcal{A}_I^{\mathbf{P},1} \\ a \in \mathscr{B}}} \mathbf{D}_{ia} \mathbf{P}_{ia} + \sum_{\substack{i \in \mathcal{A}_I^{\mathbf{P},*} \\ a \in \mathscr{B}}} \mathbf{D}_{ia} \mathbf{P}_{ia} + \sum_{\substack{i \in \mathcal{A}_O^{\mathbf{P},0} \\ a \in \mathscr{B}}} \mathbf{D}_{ia} \mathbf{P}_{ia} + \sum_{\substack{i \in \mathcal{A}_O^{\mathbf{P},1} \\ a \in \mathscr{B}}} \mathbf{D}_{ia} \mathbf{P}_{ia} \tag{5}$$

$$= 0 + \sum_{\substack{i \in \mathcal{A}_I^{\mathbf{P},1} \\ a \in \mathscr{B}}} \mathbf{D}_{ia} \mathbf{P}_{ia} + \sum_{\substack{i \in \mathcal{A}_I^{\mathbf{P},*} \\ a \in \mathscr{B}}} \mathbf{D}_{ia} \mathbf{P}_{ia} + 0 + \sum_{\substack{i \in \mathcal{A}_O^{\mathbf{P},1} \\ a \in \mathscr{B}}} \mathbf{D}_{ia} \mathbf{P}_{ia} \tag{6}$$

$$= \sum_{i \in \mathcal{A}_I^{\mathbf{P},1}} \mathbf{D}_{i, a=\tau(i) \neq \tau^*(i)} + \sum_{i \in \mathcal{A}_I^{\mathbf{P},*}} \mathbf{D}_{i, a=\tau^*(i)} + \sum_{i \in \mathcal{A}_O^{\mathbf{P},1}} \mathbf{D}_{i, a=\tau(i) \neq \tau^*(i)} \tag{7}$$

$$\geq \sum_{i \in \mathcal{A}_I^{\mathbf{P},1}} \mathbf{D}_{i, a=\tau^*(i)} + \sum_{i \in \mathcal{A}_I^{\mathbf{P},*}} \mathbf{D}_{i, a=\tau^*(i)} + \sum_{i \in \mathcal{A}_I^{\mathbf{P},0}} \mathbf{D}_{i, a=\tau^*(i)} = F_u(\mathbf{P}^*). \tag{8}$$

From Eq. (7) to Eq. (8), it holds due to the unary consistency and distinguishability expressed in proposition 2 and 3 in our manuscript. Note that, the equation holds if and only if $\mathcal{A}_I^{\mathbf{P},1} = \varnothing, \mathcal{A}_O^{\mathbf{P},1} = \varnothing$, which means that $\mathbf{P} = \mathbf{P}^*$. □

Next, we demonstrate that it also holds for the pairwise potential $F_p(\mathbf{P})$. We begin with $F_{p_1}(\mathbf{P}) = ||\mathbf{A} - \mathbf{PBP}^\mathrm{T}||_\mathcal{E}^2$ and prove that $\forall \mathbf{P} \in \mathcal{P}_k, F_{p_1}(\mathbf{P}) \geq F_{p_1}(\mathbf{P}^*)$. The proof is mainly organized based on the basic set operations, such as intersection, union, complement, Cartesian product, *etc.*.

*Proof.* First, we divide $F_{p_1}(\mathbf{P})$ into two disjoint parts according to the index sets defined above,

$$F_{p_1}(\mathbf{P}) = ||\mathbf{A} - \mathbf{PBP}^\mathrm{T}||_\mathcal{E}^2 = \sum_{ij} \mathcal{E}_{ij}(\mathbf{A}_{ij} - \sum_{a,b} \mathbf{P}_{ia}\mathbf{B}_{ab}\mathbf{P}_{jb})^2 = \sum_{ij} \mathcal{E}_{ij}(\mathbf{A}_{ij} - \mathbf{P}_{i,:}\mathbf{BP}_{j,:}^\mathrm{T})^2 \qquad (9)$$

$$= \sum_{i \in \mathscr{A}_I, j \in \mathscr{A}_I} \mathcal{E}_{ij}||\mathbf{A}_{ij} - \mathbf{P}_{i,:}\mathbf{BP}_{j,:}^\mathrm{T}||^2 \qquad (10)$$

$$+ \sum_{\substack{i \in \mathscr{A}_O, j \in \mathscr{A} \\ \text{or } i \in \mathscr{A}, j \in \mathscr{A}_O}} \mathcal{E}_{ij}||\mathbf{A}_{ij} - \mathbf{P}_{i,:}\mathbf{BP}_{j,:}^\mathrm{T}||^2. \qquad (11)$$

$$(10) = \sum_{\substack{i \in \mathcal{A}_I^{\mathbf{P},0} \\ j \in \mathscr{A}_I}} \mathcal{E}_{ij}||\mathbf{A}_{ij} - \mathbf{P}_{i,:}\mathbf{BP}_{j,:}^\mathrm{T}||^2 + \sum_{\substack{i \in \mathcal{A}_I^{\mathbf{P},1} \\ j \in \mathscr{A}_I}} \mathcal{E}_{ij}||\mathbf{A}_{ij} - \mathbf{P}_{i,:}\mathbf{BP}_{j,:}^\mathrm{T}||^2 + \sum_{\substack{i \in \mathcal{A}_I^{\mathbf{P},*} \\ j \in \mathscr{A}_I}} \mathcal{E}_{ij}||\mathbf{A}_{ij} - \mathbf{P}_{i,:}\mathbf{BP}_{j,:}^\mathrm{T}||^2. \qquad (12)$$

$$(11) = \sum_{\substack{i \in \mathcal{A}_O^{\mathbf{P},0}, j \in \mathscr{A} \\ \text{or } i \in \mathscr{A}, j \in \mathcal{A}_O^{\mathbf{P},0}}} \mathcal{E}_{ij}||\mathbf{A}_{ij} - \mathbf{P}_{i,:}\mathbf{BP}_{j,:}^\mathrm{T}||^2 + \sum_{\substack{i \in \mathcal{A}_O^{\mathbf{P},1}, j \in \mathscr{A}/\mathcal{A}_O^{\mathbf{P},0} \\ \text{or } i \in \mathscr{A}/\mathcal{A}_O^{\mathbf{P},0}, j \in \mathcal{A}_O^{\mathbf{P},1}}} \mathcal{E}_{ij}||\mathbf{A}_{ij} - \mathbf{P}_{i,:}\mathbf{BP}_{j,:}^\mathrm{T}||^2. \qquad (13)$$

Due to the proposition 1 above, we can see that,

$$F_{p_1}(\mathbf{P}^*) = \sum_{\substack{i \in \mathcal{A}_I^{\mathbf{P}^*,*} \\ j \in \mathscr{A}_I}} \mathcal{E}_{ij}||\mathbf{A}_{ij} - \mathbf{P}_{i,:}^*\mathbf{B}\mathbf{P}_{j,:}^{*\mathrm{T}}||^2 + \sum_{\substack{i \in \mathcal{A}_O^{\mathbf{P}^*,0}, j \in \mathscr{A} \\ \text{or } i \in \mathscr{A}, j \in \mathcal{A}_O^{\mathbf{P}^*,0}}} \mathcal{E}_{ij}||\mathbf{A}_{ij} - \mathbf{P}_{i,:}^*\mathbf{B}\mathbf{P}_{j,:}^{*\mathrm{T}}||^2 \qquad (14)$$

$$= \sum_{\substack{i \in \mathcal{A}_I^{\mathbf{P}^*,*} \\ j \in \mathcal{A}_I^{\mathbf{P}^*,*}}} \mathcal{E}_{ij}||\mathbf{A}_{ij} - \mathbf{B}_{\tau^*(i)\tau^*(j)}||^2 + \sum_{\substack{i \in \mathcal{A}_O^{\mathbf{P}^*,0}, j \in \mathscr{A} \\ \text{or } i \in \mathscr{A}, j \in \mathcal{A}_O^{\mathbf{P}^*,0}}} \mathcal{E}_{ij}||\mathbf{A}_{ij} - 0||^2 \qquad (15)$$

$$= \sum_{\substack{i \in \mathscr{A}_I \\ j \in \mathscr{A}_I}} \mathcal{E}_{ij}||\mathbf{A}_{ij} - \mathbf{B}_{\tau^*(i)\tau^*(j)}||^2 + \sum_{\substack{i \in \mathscr{A}_O, j \in \mathscr{A} \\ \text{or } i \in \mathscr{A}, j \in \mathscr{A}_O}} \mathcal{E}_{ij}||\mathbf{A}_{ij} - 0||^2. \qquad (16)$$

Then, for $\forall \mathbf{P} \in \mathcal{P}_k$ and $\mathbf{P} \neq \mathbf{P}^*$, we reorganize Eq. (10) and Eq. (11) based on the index sets.

$$(10) = \sum_{\substack{i \in \mathcal{A}_I^{\mathbf{P},0} \\ j \in \mathscr{A}_I}} \mathcal{E}_{ij}||\mathbf{A}_{ij} - 0||^2 + \sum_{\substack{i \in \mathcal{A}_I^{\mathbf{P},1} \\ j \in \mathcal{A}_I^{\mathbf{P},0}}} \mathcal{E}_{ij}||\mathbf{A}_{ij} - 0||^2 + \sum_{\substack{i \in \mathcal{A}_I^{\mathbf{P},1} \\ j \in \mathscr{A}_I/\mathcal{A}_I^{\mathbf{P},0}}} \mathcal{E}_{ij}||\mathbf{A}_{ij} - \mathbf{B}_{\tau(i)\tau(j)}||^2 \qquad (17)$$

$$+ \sum_{\substack{i \in \mathcal{A}_I^{\mathbf{P},*} \\ j \in \mathcal{A}_I^{\mathbf{P},0}}} \mathcal{E}_{ij}||\mathbf{A}_{ij} - 0||^2 + \sum_{\substack{i \in \mathcal{A}_I^{\mathbf{P},*} \\ j \in \mathcal{A}_I^{\mathbf{P},1}}} \mathcal{E}_{ij}||\mathbf{A}_{ij} - \mathbf{B}_{\tau(i)\tau(j)}||^2 + \sum_{\substack{i \in \mathcal{A}_I^{\mathbf{P},*} \\ j \in \mathcal{A}_I^{\mathbf{P},*}}} \mathcal{E}_{ij}||\mathbf{A}_{ij} - \mathbf{B}_{\tau(i)\tau(j)}||^2 \qquad (18)$$

$$= \sum_{\substack{i \in \mathcal{A}_I^{\mathbf{P},0} \\ j \in \mathscr{A}_I/\mathcal{A}_I^{\mathbf{P},0}}} \mathcal{E}_{ij}||\mathbf{A}_{ij}||^2 + \sum_{\substack{i \in \mathcal{A}_I^{\mathbf{P},0} \\ j \in \mathcal{A}_I^{\mathbf{P},0}}} \mathcal{E}_{ij}||\mathbf{A}_{ij}||^2 + \sum_{\substack{i \in \mathscr{A}_I/\mathcal{A}_I^{\mathbf{P},0} \\ j \in \mathcal{A}_I^{\mathbf{P},0}}} \mathcal{E}_{ij}||\mathbf{A}_{ij}||^2 \qquad (19)$$

$$+ \sum_{\substack{i \in \mathcal{A}_I^{\mathbf{P},1} \\ j \in \mathscr{A}_I/\mathcal{A}_I^{\mathbf{P},0}}} \mathcal{E}_{ij}||\mathbf{A}_{ij} - \mathbf{B}_{\tau(i)\tau(j)}||^2 + \sum_{\substack{i \in \mathcal{A}_I^{\mathbf{P},*} \\ j \in \mathcal{A}_I^{\mathbf{P},1}}} \mathcal{E}_{ij}||\mathbf{A}_{ij} - \mathbf{B}_{\tau(i)\tau(j)}||^2 + \sum_{\substack{i \in \mathcal{A}_I^{\mathbf{P},*} \\ j \in \mathcal{A}_I^{\mathbf{P},*}}} \mathcal{E}_{ij}||\mathbf{A}_{ij} - \mathbf{B}_{\tau(i)\tau(j)}||^2 \qquad (20)$$

$$= \sum_{\substack{i \in \mathcal{A}_I^{\mathbf{P},0} \\ j \in \mathscr{A}_I/\mathcal{A}_I^{\mathbf{P},0}}} \mathcal{E}_{ij}||\mathbf{A}_{ij}||^2 + \sum_{\substack{i \in \mathscr{A}_I/\mathcal{A}_I^{\mathbf{P},0} \\ j \in \mathcal{A}_I^{\mathbf{P},0}}} \mathcal{E}_{ij}||\mathbf{A}_{ij}||^2 + \sum_{\substack{i \in \mathcal{A}_I^{\mathbf{P},0} \\ j \in \mathcal{A}_I^{\mathbf{P},0}}} \mathcal{E}_{ij}||\mathbf{A}_{ij}||^2 + \sum_{\substack{i \in \mathcal{A}_I^{\mathbf{P},*} \cup \mathcal{A}_I^{\mathbf{P},1} \\ j \in \mathcal{A}_I^{\mathbf{P},*} \cup \mathcal{A}_I^{\mathbf{P},1}}} \mathcal{E}_{ij}||\mathbf{A}_{ij} - \mathbf{B}_{\tau(i)\tau(j)}||^2. \qquad (21)$$

$$(11) = \sum_{\substack{i\in\mathcal{A}_O^{\mathbf{P},0}, j\in\mathscr{A} \\ \text{or } i\in\mathscr{A}, j\in\mathcal{A}_O^{\mathbf{P},0}}} \mathcal{E}_{ij}\|\mathbf{A}_{ij}-0\|^2 + \sum_{\substack{i\in\mathcal{A}_O^{\mathbf{P},1}, j\in\mathcal{A}_I^{\mathbf{P},0} \\ \text{or } i\in\mathcal{A}_I^{\mathbf{P},0}, j\in\mathcal{A}_O^{\mathbf{P},1}}} \mathcal{E}_{ij}\|\mathbf{A}_{ij}-0\|^2 \qquad (22)$$

$$+ \sum_{\substack{i\in\mathcal{A}_O^{\mathbf{P},1} \\ j\in\mathscr{A}_I/\mathcal{A}_I^{\mathbf{P},0}}} \mathcal{E}_{ij}\|\mathbf{A}_{ij}-\mathbf{B}_{\tau(i)\tau(j)}\|^2 + \sum_{\substack{i\in\mathscr{A}_I/\mathcal{A}_I^{\mathbf{P},0} \\ j\in\mathcal{A}_O^{\mathbf{P},1}}} \mathcal{E}_{ij}\|\mathbf{A}_{ij}-\mathbf{B}_{\tau(i)\tau(j)}\|^2 + \sum_{\substack{i\in\mathcal{A}_O^{\mathbf{P},1} \\ j\in\mathcal{A}_O^{\mathbf{P},1}}} \mathcal{E}_{ij}\|\mathbf{A}_{ij}-\mathbf{B}_{\tau(i)\tau(j)}\|^2. \qquad (23)$$

At last, we prove $F_{p_1}(\mathbf{P}) - F_{p_1}(\mathbf{P}^*) \geq 0$. To make it more clear, we still reorganize these terms by their index sets.

$$F_{p_1}(\mathbf{P}) - F_{p_1}(\mathbf{P}^*)$$

$$= [\sum_{\substack{i\in\mathcal{A}_I^{\mathbf{P},0} \\ j\in\mathscr{A}_I/\mathcal{A}_I^{\mathbf{P},0}}} \mathcal{E}_{ij}\|\mathbf{A}_{ij}\|^2 + \sum_{\substack{i\in\mathscr{A}_I/\mathcal{A}_I^{\mathbf{P},0} \\ j\in\mathcal{A}_I^{\mathbf{P},0}}} \mathcal{E}_{ij}\|\mathbf{A}_{ij}\|^2 + \sum_{\substack{i\in\mathcal{A}_I^{\mathbf{P},0} \\ j\in\mathcal{A}_I^{\mathbf{P},0}}} \mathcal{E}_{ij}\|\mathbf{A}_{ij}\|^2] + \sum_{\substack{i\in\mathcal{A}_I^{\mathbf{P},*}\cup\mathcal{A}_I^{\mathbf{P},1} \\ j\in\mathcal{A}_I^{\mathbf{P},*}\cup\mathcal{A}_I^{\mathbf{P},1}}} \mathcal{E}_{ij}\|\mathbf{A}_{ij}-\mathbf{B}_{\tau(i)\tau(j)}\|^2 \qquad (24)$$

$$- [\sum_{\substack{i\in\mathcal{A}_I^{\mathbf{P},0} \\ j\in\mathscr{A}_I/\mathcal{A}_I^{\mathbf{P},0}}} + \sum_{\substack{i\in\mathscr{A}_I/\mathcal{A}_I^{\mathbf{P},0} \\ j\in\mathcal{A}_I^{\mathbf{P},0}}} + \sum_{\substack{i\in\mathcal{A}_I^{\mathbf{P},0} \\ j\in\mathcal{A}_I^{\mathbf{P},0}}}]\mathcal{E}_{ij}\|\mathbf{A}_{ij}-\mathbf{B}_{\tau^*(i)\tau^*(j)}\|^2 - \sum_{\substack{i\in\mathcal{A}_I^{\mathbf{P},*}\cup\mathcal{A}_I^{\mathbf{P},1} \\ j\in\mathcal{A}_I^{\mathbf{P},*}\cup\mathcal{A}_I^{\mathbf{P},1}}}\mathcal{E}_{ij}\|\mathbf{A}_{ij}-\mathbf{B}_{\tau^*(i)\tau^*(j)}\|^2 \qquad (25)$$

$$+ [\sum_{\substack{i\in\mathcal{A}_O^{\mathbf{P},1} \\ j\in\mathscr{A}_I/\mathcal{A}_I^{\mathbf{P},0}}}\mathcal{E}_{ij}\|\mathbf{A}_{ij}-\mathbf{B}_{\tau(i)\tau(j)}\|^2 + \sum_{\substack{i\in\mathscr{A}_I/\in\mathcal{A}_I^{\mathbf{P},0} \\ j\in\mathcal{A}_O^{\mathbf{P},1}}}\mathcal{E}_{ij}\|\mathbf{A}_{ij}-\mathbf{B}_{\tau(i)\tau(j)}\|^2 + \sum_{\substack{i\in\mathcal{A}_O^{\mathbf{P},1} \\ j\in\mathcal{A}_O^{\mathbf{P},1}}}\mathcal{E}_{ij}\|\mathbf{A}_{ij}-\mathbf{B}_{\tau(i)\tau(j)}\|^2] \qquad (26)$$

$$- [\sum_{\substack{i\in\mathcal{A}_O^{\mathbf{P},1} \\ j\in\mathscr{A}_I/\mathcal{A}_I^{\mathbf{P},0}}}\mathcal{E}_{ij}\|\mathbf{A}_{ij}\|^2 + \sum_{\substack{i\in\mathscr{A}_I/\mathcal{A}_I^{\mathbf{P},0} \\ j\in\mathcal{A}_O^{\mathbf{P},1}}}\mathcal{E}_{ij}\|\mathbf{A}_{ij}\|^2 + \sum_{\substack{i\in\mathcal{A}_O^{\mathbf{P},1} \\ j\in\mathcal{A}_O^{\mathbf{P},1}}}\mathcal{E}_{ij}\|\mathbf{A}_{ij}\|^2] \qquad (27)$$

$$= [\sum_{\substack{i\in\mathcal{A}_I^{\mathbf{P},0} \\ j\in\mathscr{A}_I/\mathcal{A}_I^{\mathbf{P},0}}} + \sum_{\substack{i\in\mathscr{A}_I/\mathcal{A}_I^{\mathbf{P},0} \\ j\in\mathcal{A}_I^{\mathbf{P},0}}} + \sum_{\substack{i\in\mathcal{A}_I^{\mathbf{P},0} \\ j\in\mathcal{A}_I^{\mathbf{P},0}}}]\mathcal{E}_{ij}\|\mathbf{A}_{ij}\|^2 - [\sum_{\substack{i\in\mathcal{A}_O^{\mathbf{P},1} \\ j\in\mathscr{A}_I/\mathcal{A}_I^{\mathbf{P},0}}} + \sum_{\substack{i\in\mathscr{A}_I/\mathcal{A}_I^{\mathbf{P},0} \\ j\in\mathcal{A}_O^{\mathbf{P},1}}} + \sum_{\substack{i\in\mathcal{A}_O^{\mathbf{P},1} \\ j\in\mathcal{A}_O^{\mathbf{P},1}}}]\mathcal{E}_{ij}\|\mathbf{A}_{ij}\|^2 \qquad (28)$$

$$+ \sum_{\substack{i\in\mathcal{A}_I^{\mathbf{P},*}\cup\mathcal{A}_I^{\mathbf{P},1} \\ j\in\mathcal{A}_I^{\mathbf{P},*}\cup\mathcal{A}_I^{\mathbf{P},1}}}\mathcal{E}_{ij}\|\mathbf{A}_{ij}-\mathbf{B}_{\tau(i)\tau(j)}\|^2 - \sum_{\substack{i\in\mathcal{A}_I^{\mathbf{P},*}\cup\mathcal{A}_I^{\mathbf{P},1} \\ j\in\mathcal{A}_I^{\mathbf{P},*}\cup\mathcal{A}_I^{\mathbf{P},1}}}\mathcal{E}_{ij}\|\mathbf{A}_{ij}-\mathbf{B}_{\tau^*(i)\tau^*(j)}\|^2 \qquad (29)$$

$$+ [\sum_{\substack{i\in\mathcal{A}_O^{\mathbf{P},1} \\ j\in\mathscr{A}_I/\mathcal{A}_I^{\mathbf{P},0}}} + \sum_{\substack{i\in\mathscr{A}_I/\in\mathcal{A}_I^{\mathbf{P},0} \\ j\in\mathcal{A}_O^{\mathbf{P},1}}} + \sum_{\substack{i\in\mathcal{A}_O^{\mathbf{P},1} \\ j\in\mathcal{A}_O^{\mathbf{P},1}}}]\mathcal{E}_{ij}\|\mathbf{A}_{ij}-\mathbf{B}_{\tau(i)\tau(j)}\|^2 - [\sum_{\substack{i\in\mathcal{A}_I^{\mathbf{P},0} \\ j\in\mathscr{A}_I/\mathcal{A}_I^{\mathbf{P},0}}} + \sum_{\substack{i\in\mathscr{A}_I/\mathcal{A}_I^{\mathbf{P},0} \\ j\in\mathcal{A}_I^{\mathbf{P},0}}} + \sum_{\substack{i\in\mathcal{A}_I^{\mathbf{P},0} \\ j\in\mathcal{A}_I^{\mathbf{P},0}}}]\mathcal{E}_{ij}\|\mathbf{A}_{ij}-\mathbf{B}_{\tau^*(i)\tau^*(j)}\|^2.$$
$$(30)$$

We can see that Eq. (28) $\geq 0$, Eq. (29) $\geq 0$, Eq. (30) $\geq 0$ due to the proposition 2, 3 and 4 in our manuscript. Note that, the equation holds if and only if $\mathbf{P} = \mathbf{P}^*$. $\square$

Similarly, we can prove that $F_{p_2}(\mathbf{P}) \geq F_{p_2}(\mathbf{P}^*)$. Finally, we have proved that $F(\mathbf{P}) \geq F(\mathbf{P}^*)$ because $F = \lambda_1 F_u + \lambda_2(F_{p_1} + F_{p_2})$ with $\lambda_1, \lambda_2 \geq 0$. $\square$

There is one more interesting conclusion we can summarize here. For general graph matching on simple graphs without outliers or with outliers arising in only one graph (*e.g.*, $\mathcal{G}'$), we can see that the index sets satisfy $\mathcal{A}_I^{\mathbf{P},0} = \varnothing, \mathcal{A}_O^{\mathbf{P},0} = \varnothing, \mathcal{A}_I^{\mathbf{P},0} = \varnothing, \mathcal{A}_O^{\mathbf{P},1} = \varnothing$. Therefore, it naturally holds that Eq. (28) = 0, Eq. (30) = 0. And Eq. (29) can be simplified as

$$(29) = \sum_{\substack{i\in\mathcal{A}_I^{\mathbf{P},1}, j\in\mathcal{A}_I^{\mathbf{P},*}\cup\mathcal{A}_I^{\mathbf{P},1} \\ \text{or } i\in\mathcal{A}_I^{\mathbf{P},*}\cup\mathcal{A}_I^{\mathbf{P},1}, j\in\mathcal{A}_I^{\mathbf{P},1}}} \mathcal{E}_{ij}\left(\|\mathbf{A}_{ij}-\mathbf{B}_{\tau(i)\tau(j)}\|^2 - \|\mathbf{A}_{ij}-\mathbf{B}_{\tau^*(i)\tau^*(j)}\|^2\right). \qquad (31)$$

For simple graphs, it usually satisfies that the discrepancy $\|\mathbf{A}_{ij}-\mathbf{B}_{\tau^*(i)\tau^*(j)}\|$ between $\mathbf{A}_{ij}$ and the ideally matched one $\mathbf{B}_{\tau^*(i)\tau^*(j)}$ is smaller than $\|\mathbf{A}_{ij}-\mathbf{B}_{\tau(i)\tau(j)}\|$, which directly ensures that Eq. (31)$\geq 0$. It means that, for simple graphs, the choices for $\mathcal{E}, \mathbf{A}, \mathbf{B}$ can be more flexible. Furthermore, it demonstrates why the GM methods suitable for simple graphs can not be directly applied on complicated graphs consisting of numerous outliers arising in both graphs.

## 2. More comparison results

In this section, we present more results of the graph matching methods on the three datasets used in our manuscript.

### 2.1. More results on PASCAL dataset

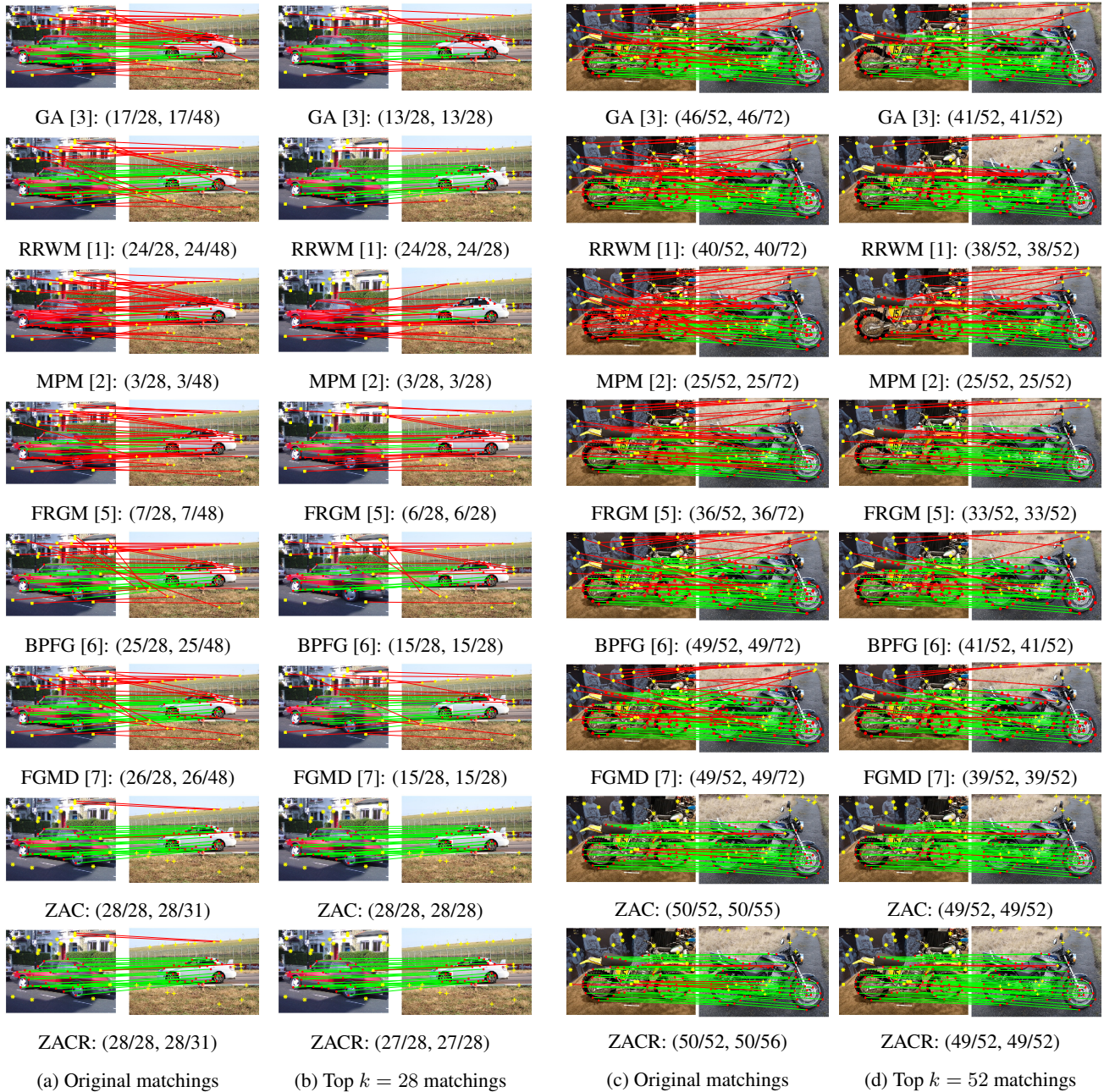

Figure 3: Examples of graph matching with inliers (red dots) and outliers (yellow signs) on the car images (left two columns) and motorbike images (right two columns) in PASCAL dataset. For each method, the recall and precision are computed with their original matchings and top $k = \#inliers$ matchings, respectively.

## 2.2. More results on VGG dataset

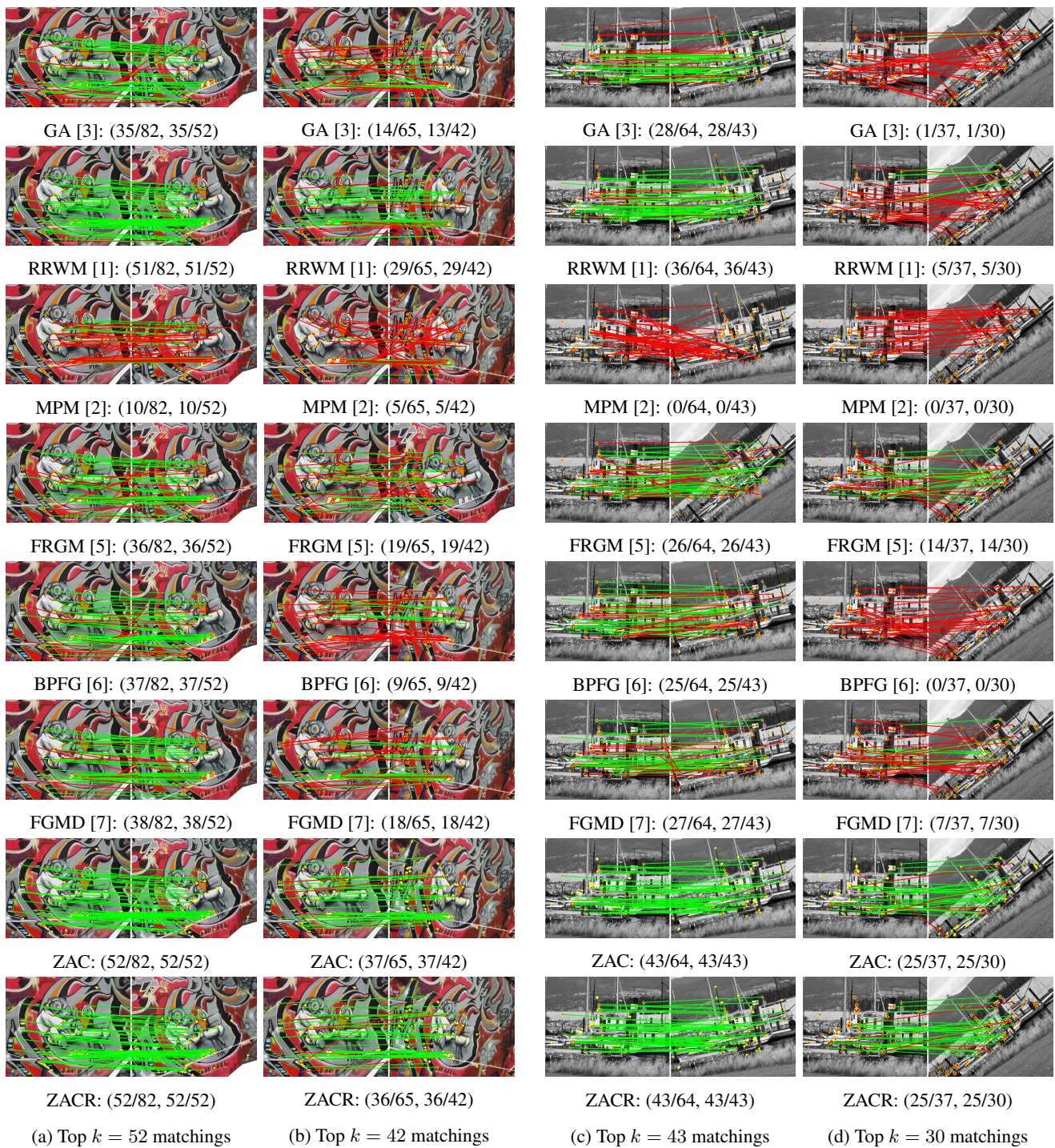

Figure 4: Examples of graph matching under varying viewpoint (left two columns) and zoom+rotation (right two columns) on the "graf" images and "boat" images in VGG dataset. For each method, the recall and precision are computed with their top $k = \lfloor \{0.5, 0.4\} \times \min\{m, n\} \rfloor$ matchings, respectively.

## 2.3. More results of deformable graph matching

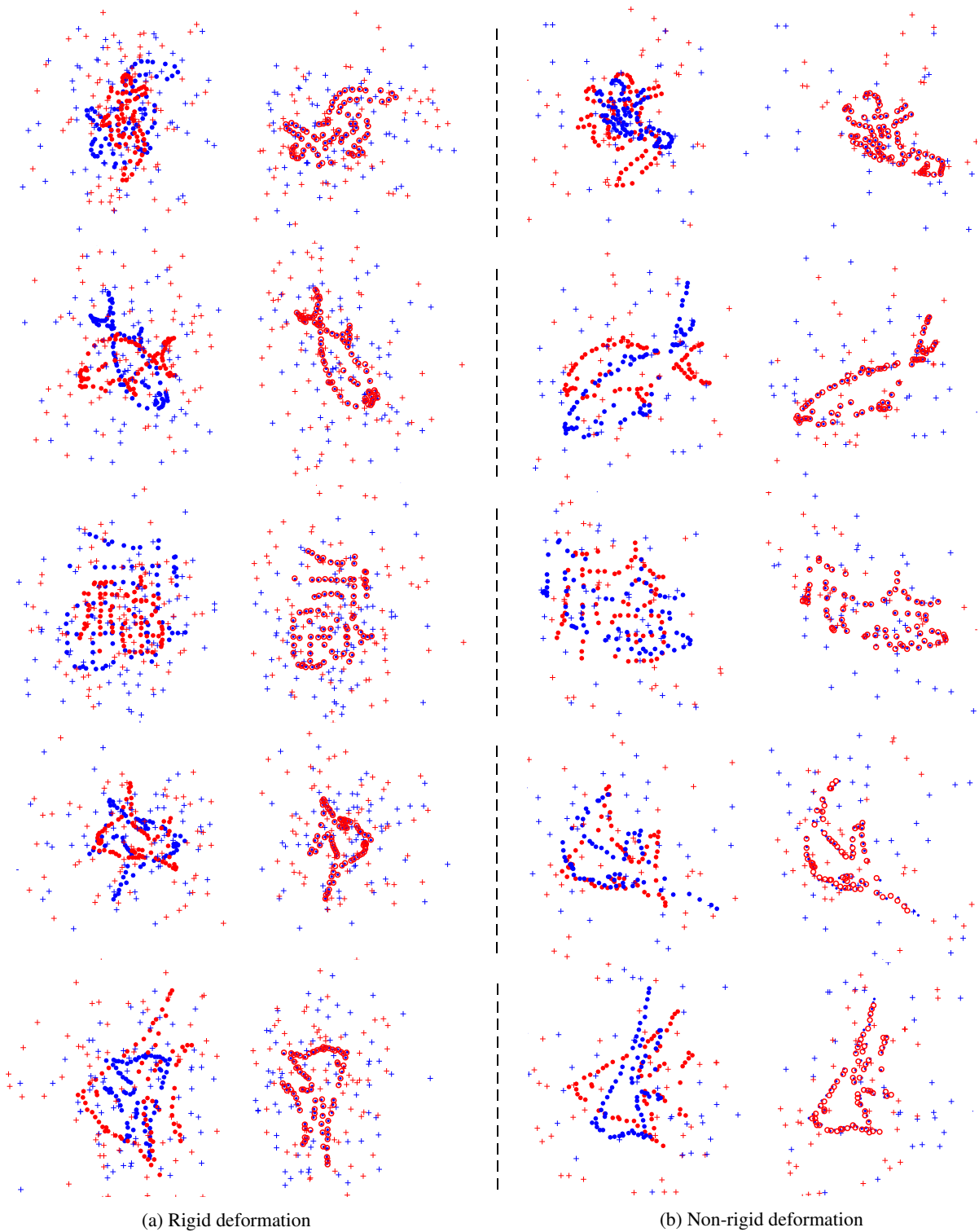

(a) Rigid deformation  (b) Non-rigid deformation

Figure 5: Examples of deformable graph matching by our method ZAC on widely used shape templates (also used in [5, 4], *etc*.). The graph pairs are disturbed by geometric deformations, noises, missing points and outliers.



\end{document}